\documentclass[12pt,reqno,twoside]{amsart}
\usepackage{amsfonts,amsmath,amssymb}
\usepackage{mathrsfs,mathtools}
\usepackage{enumerate}
\usepackage{hyperref}
\usepackage{esint}
\usepackage{graphicx}
\usepackage{bm}
\usepackage{commath}
\usepackage{esint}
\usepackage{placeins}
\usepackage{flafter}
\usepackage{algorithm}
\usepackage{algpseudocode}
\usepackage{tikz}
\usepackage[textsize=small]{todonotes}
\usepackage[dvips,bottom=1.4in,right=1in,top=1in, left=1in]{geometry}
\usepackage[latin1]{inputenc}
\usepackage{comment}
\usepackage{float}
\usepackage{setspace}

\DeclareMathAlphabet{\mathpzc}{OT1}{pzc}{m}{it}
\numberwithin{equation}{section}

\makeatletter
\def\eqnarray{\stepcounter{equation}\let\@currentlabel=\theequation
	\global\@eqnswtrue
	\tabskip\@centering\let\\=\@eqncr
	$$\halign to \displaywidth\bgroup\hfil\global\@eqcnt\z@
	$\displaystyle\tabskip\z@{##}$&\global\@eqcnt\@ne
	\hfil$\displaystyle{{}##{}}$\hfil
	&\global\@eqcnt\tw@ $\displaystyle{##}$\hfil
	\tabskip\@centering&\llap{##}\tabskip\z@\cr}
\def\endeqnarray{\@@eqncr\egroup
	\global\advance\c@equation\m@ne$$\global\@ignoretrue}

\newtheorem{theorem}{Theorem}[section]

\newtheorem{definition}[theorem]{Def{}inition}

\newtheorem{lemma}[theorem]{Lemma}

\numberwithin{equation}{section}

\usepackage[textsize=small]{todonotes}
\newcommand{\DS}[1]{{\color{magenta}~\textsf{#1}}}

\DeclareMathOperator{\argmin}{argmin}

\title[GNEP Based Dynamic Segmentation and Motion Estimation]{GNEP Based Dynamic Segmentation and Motion Estimation for Neuromorphic Imaging} 
\date{\today}

\thanks{
	This work is partially supported by NSF grant DMS-2110263 and the Air Force Office of Scientific Research (AFOSR) under Award NO: FA9550-22-1-0248. 
}

\author{Harbir Antil}
\address{H. Antil. The Center for Mathematics and Artificial Intelligence
	(CMAI) and Department of Mathematical Sciences, George Mason University,
	Fairfax, VA 22030, USA.}
\email{hantil@gmu.edu}
\author{David Sayre}
\address{D. Sayre. The Center for Mathematics and Artificial Intelligence
	(CMAI) and Department of Mathematical Sciences, George Mason University,
	Fairfax, VA 22030, USA.}
\email{dsayre@gmu.edu}

\begin{document}				
	
	\begin{abstract}	
		This paper explores the application of event-based cameras in the domains of image segmentation and motion estimation. These cameras offer a groundbreaking technology by capturing visual information as a continuous stream of asynchronous events, departing from the conventional frame-based image acquisition. We introduce a Generalized Nash Equilibrium based framework that leverages the temporal and spatial information derived from the event stream to carry out segmentation and velocity estimation. To establish the theoretical foundations, we derive an existence criteria and propose a multi-level optimization method for calculating equilibrium. The efficacy of this approach is shown through a series of experiments.
	\end{abstract}

	
	\maketitle

	\section{Introduction} \label{s:Intro}
	Event (Neuromorphic) cameras are a class of biologically inspired sensors that employ an asynchronous sampling mechanism to record data based on the change in light intensity at each pixel \cite{antil2023bilevel,gallego2020event,stoffregen2018simultaneous}. Specifically, an event is triggered at a pixel if and only if the change in light intensity at that pixel exceeds a predefined threshold. As a consequence, no events are recorded if there is no change to the scene.
	In contrast, if the scene is dynamic, either due to camera movement or due to the movement of an object in the scene each pixel of the event camera records intensity changes with a temporal resolution on the order of microseconds \cite{gallego2020event,stoffregen2019event}. Mathematically, an event camera can be modeled as a spatio-temporal point process, where the occurrence of an event at a pixel is a binary-valued function of time \cite{stoffregen2019event}. In \cite{antil2023bilevel}, the authors introduced a mathematically rigorous bilevel optimization model, including a complete analysis, to carry out reconstructions from the event data. This model can use both a combination of traditional frames and event data. Figure \ref{f:nucamera} illustrates the differences between traditional and neuromorphic cameras.
	
	\begin{figure}[!htb]
		\centering
		\includegraphics[width=0.9\textwidth]{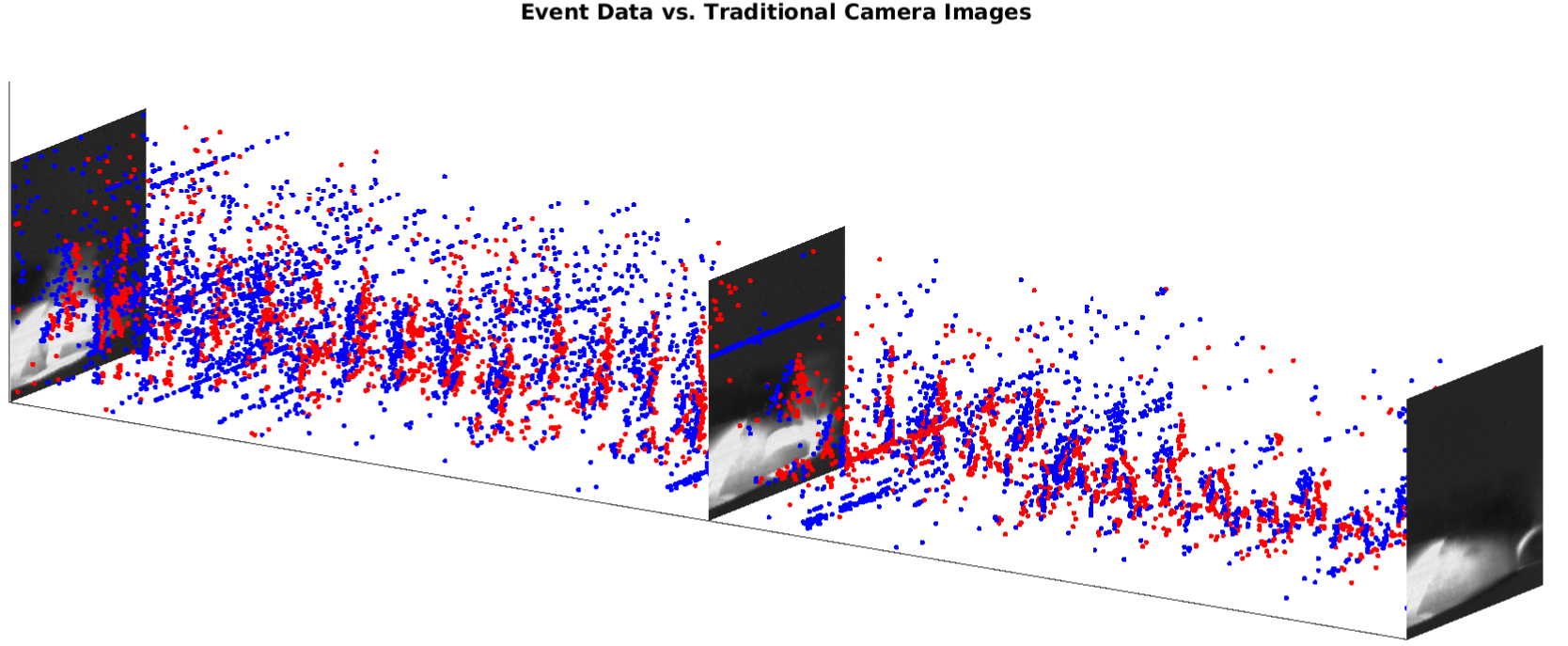}
		\caption{
			Comparison of the output of a standard frame-based camera (3 frames are shown)
			and event camera (dots indicate the events). The standard camera 
			outputs frames at a fixed rate, thus sending redundant information when there 
			is no motion in the scene. However, the event based camera captures the 
			change in scene. For a model that can utilize all this information to carry 
			out high quality reconstructions, we refer to \cite{antil2023bilevel}.
		}
		\label{f:nucamera}
	\end{figure}
	
	Event image segmentation has emerged as a topic of interest, offering promising solutions to the challenging task of segmenting dynamic scenes captured by event-based cameras. Under the assumption of constant illumination, event occurrences in event-based cameras can be attributed to the relative motion between the camera and the observed scene \cite{gallego2019focus,stoffregen2019event2}. As events are primarily indicative of changes in brightness, they predominantly arise from scene edges, including contours, texture variations, and similar visual characteristics \cite{zhou2021event}. In the case of a stationary camera, events solely originate from moving objects within the scene \cite{parameshwara2021spikems}. However, when the camera itself undergoes motion, events are generated by both the edges of moving objects and the motion-induced edges of stationary objects \cite{liu2020globally}. This interplay between the camera's motion and the inherent dynamics of the scene adds complexity to event-based image segmentation, necessitating robust approaches to effectively distinguish between these sources of event occurrences. Various methods have been proposed to solve these problems including learning based \cite{parameshwara2021spikems,zihao2018unsupervised}, probabilistic \cite{zhu2017event} and contrast maximization models \cite{gallego2018unifying,glover2016event,stoffregen2019event,stoffregen2018simultaneous,stoffregen2019event2,vasco2017independent,zhou2021event}.
	
	Contrast maximization has emerged as a promising technique for event image segmentation. At its core, this approach focuses on finding a motion parameter such that the contrast is maximized when events are projected to an image plane.
	In the case of multiple objects a secondary task of discerning which events belong to which object must be undertaken \cite{stoffregen2019event,zhou2021event}. While these models/methods represent a leap forward but they are largely heuristic in nature and various questions remain about these models such as the existence of solutions, initialization of parameters, and optimization procedures to find solutions. 
	
	In this paper we provide alternative formulation to the previous methods by introducing a game theoretic model based on a Stackelberg Game. The presented approach is mathematically and algorithmically rigorous. A Stackelberg game is a type of non-cooperative game in game theory, named after the German economist Heinrich von Stackelberg. In a Stackelberg game, players make their decisions sequentially, with one player known as the leader or the Stackelberg player, acting first, and the other player or players, known as the followers, making their decisions after observing the leader's choice \cite{hu2015multi,julien2018stackelberg,kulkarni2015existence,pang2009quasi}, see also \cite{SCui_UVShanbhag_2023a} for  complexity guarantees in the context of multi-leader and multi-follower games. Stackelberg games are often modeled mathematically as a Nash Equilibrium Problem or a Generalized Nash Equilibrium problem \cite{aussel2020short,hu2015multi,wang2017distributed}. The Generalized Nash Equilibrium problem (GNEP) provides a mathematical framework for modeling a broad range of real-world scenarios involving strategic interactions among multiple autonomous agents. Specifically, GNEPs capture situations in which agents pursue individual objectives that are interdependent, thereby necessitating the consideration of others' decisions in determining their own optimal actions \cite{aussel2008generalized,facchinei2010generalized}. 
	
	In contrast to previous work we define the motion parameter as a function of the neuromorphic events versus defining the events as a function of their associated motion parameter. Our players will adopt a strategic approach by assigning a confidence value to each event within the dataset, indicating their level of certainty regarding the association between events. Utilizing contrast maximization techniques we can evaluate the chosen strategy of each player. This allows us to identify the events that were generated by the same object and perform event based image segmentation. We enforce a restriction on the sum of confidence scores assigned to each event, limiting it to a maximum value of one. This results in a jointly constrained strategy space. In order to model this problem effectively we resort to the Generalized Nash Equilibrium model versus the Nash Equilibrium.
	In the remainder of the paper we seek to:
	\begin{enumerate}
		\item Introduce a Generalized Nash Equilibrium Problem for the segmentation of events.
		\item Establish existence of equilibrium for the proposed model.
		\item Propose and validate an $N$-level optimization method to find equilibrium.
		\item Demonstrate the effectiveness of the model.
	\end{enumerate}
	A GNEP with players that have a non-convex objective functions present several challenges. The non-convexity of the objective functions prohibits the use of traditional approaches like quasi-variational inequality methods, which depend on the convexity of the objective function as well as joint convexity of the strategy space. Instead, alternative techniques must be employed to find equilibrium solutions. One such approach is multi-level optimization, which involves decomposing the problem into multiple levels and optimizing them iteratively. This allows for tackling the non-convexity of the objective function by breaking it down into more manageable sub-problems. However, the complexity and computational demands associated with multi-level optimization make it a challenging task.
	
	\medskip
	\noindent
	{\bf Outline.} The article has been organized as follows: Section~\ref{s:not} focuses on 
	notations and preliminary work. In particular, working of GNEPs and neuromorphic
	cameras have been described. Then in section~\ref{s:pw}, we review the previous contrast 
	maximization methods. In sections \ref{s:model} and \ref{s:analysis} we introduce our model 
	as well as provide results pertaining to existence of solutions. Moreover, in section~\ref{s:numres} 
	we provide the experimental set-up for our results. Subsequently, several examples are 
	considered in section~\ref{s:nex}.

	\section{Preliminaries}
	\label{s:not}
	
	\subsection{Generalized Nash Equilibrium Problems (GNEPs)}
	\subsubsection{Overview}
	A Generalized Nash Equilibrium Problem (GNEP) is a mathematical formulation used to find the equilibrium solution in games involving multiple players, each with their own independent decision-making. The problem involves determining a set of strategies for all players such that no player can unilaterally improve their payoff by changing their strategy while taking into account the strategies chosen by the other players. We seek to find a solution that satisfies the objective function of each player while considering the strategies of the other players.
	
	\subsubsection{Notation}
	Consider a set of $j = 1, \dots, N$ agents that have control over $N_v = \sum_{j=1}^{N} n_j$ variables where $n_j$ represents the number of variables that agent $j$ controls. We denote the variables agent $j$ controls by the vector $\bm{z}^j \in \mathbb{R}^{n_j}$. 
	Additionally, we denote the vector of all decision variables as: 
	\[
	\bm{z} := 
	\begin{pmatrix}
	\bm{z}^1 & \bm{z}^2 & \cdots & \bm{z}^N
	\end{pmatrix}^\top \in \mathbb{R}^{N_v}\, .
	\]
	We extend this notation such that 
	\[
	\bm{z}^{-j} := \begin{pmatrix}
	\bm{z}^1 & \bm{z}^2 & \cdots & \bm{z}^{j-1} & \bm{z}^{j+1} & \cdots & \bm{z}^N
	\end{pmatrix}^\top \in \mathbb{R}^{N_v - n_j}\,
	\]
	represents the decision variables of all agents except agent $j$.
	The goal of each agent is to minimize their own objective function which we denote by $J^j(\bm{z}^j,\bm{z}^{-j})$ while taking into account the strategies of all other players. The feasible strategy set for player $j$ is denoted by $\bm{Z}^j(\bm{z}^{-j})$. This set captures the strategies available to player $j$ while considering the constraints imposed by the strategies chosen by the other players. We can then express the minimization problem for each player $j$ as:
	\begin{equation}\label{eq:gnep}
	\min_{\bm{z}^j \in\bm{Z}^j(\bm{z}^{-j})} J^j(\bm{z}^j,\bm{z}^{-j}) \, .
	\end{equation}
	Our goal is then to solve the above minimization problem to find an optimal 
	solution vector which we now define. 
	\begin{definition}[Generalized Nash Equilibrium \cite{facchinei2010generalized}]\label{d:gne}
		A point $\bar{\bm z}$ is said to be a (generalized Nash) equilibrium or a solution to 
		the Generalized Nash Equilibrium Problem (GNEP) if it satisfies the following:
		\begin{align*}
		\bar{\bm z} = \begin{pmatrix}
		\bar{\bm{z}}^1 & \bar{\bm{z}}^2 & \cdots & \bar{\bm{z}}^N
		\end{pmatrix}^\top
		\end{align*}
		such that
		\begin{align*}
		\bar{\bm{z}}^j \in \bm{Z}^j(\bar{\bm{z}}^{-j}) \quad \text{ for all } \quad j = 1, 2, \dots, N.
		\end{align*}
	\end{definition} 
	For more information on GNEP's and their solutions we refer the reader to 
	\cite{facchinei2010generalized,harker1991generalized,kulkarni2014shared,schiro2013solution}.
	
	\subsection{Neuromorphic Cameras}\label{s:e_dat}
	\subsubsection{Overview}
	Neuromorphic cameras consist of individual pixels that independently detect changes in light intensity in real time. These cameras sample light intensity at a rapid rate, typically on the order of microseconds, and record events when the intensity exceeds a preset threshold $c$. The use of high-speed sampling and the independent nature of neuromorphic camera pixels make them highly resistant to issues such as overexposure and image blurring. The latter has been recently 
	explored in detail in \cite{antil2023bilevel}, both from a theoretical and algorithmic point of view.
	
	\subsubsection{Data representation}\label{s:dat_rep}
	The output of a neuromorphic camera is of the form: $e = (x,y,t,p)$. If we consider 
	$k = 1,\dots,N_e$ distinct events, then we can describe each event individually 
	with the following definition:
	\[
	\mathbb{R}^{1 \times 4} \ni \bm{e}_k := (\bm{x}_k,\bm{y}_k,\bm{t}_k,\bm{p}_k).
	\]
	where subscript $k$ in $\bm x_k$ denotes the $k$-th component of vector $\bm x$.
	Moreover, $(\bm{x}_k,\bm{y}_k) \in \bm{\Omega} = \{1,2,\dots,N_x\} \times \{1,2,\dots,N_y\}$ 
	denotes the pixel location of event $k$ with $N_x \times N_y$ being the resolution of the camera. 
	Moreover, $\bm{t}_k \in \mathbb{R}$ represents the time at which event $k$ occurred 
	and $\bm{p}_k \in \{-1,1\}$ represents the polarity of event $k$.
	
	The polarity recorded by a neuromorphic camera at a given pixel $(\bm{x}_k,\bm{y}_k)$ is a 
	function of the instantaneous light intensity ($\bm{q}_k$) recorded at $(\bm{x}_k,\bm{y}_k)$ 
	at time $\bm{t}_k$. Furthermore, the reference light intensity for pixel $(\bm{x}_k,\bm{y}_k)$ 
	at time $\bm{t}_k$ is given by $\bm{q}_{ref_k}$. Using these quantities we define the 
	polarity as:
	\[
	\bm{p}_k := \begin{cases}
	+1, & \log \left( \frac{\bm{q}_{k}}{\bm{q}_{{ref_k}}}\right) \geq c\\[.2in]
	-1, & \log \left( \frac{\bm{q}_{k}}{\bm{q}_{{ref_k}}}\right) \leq -c.
	\end{cases}
	\]
	No event is recorded if $\bm{p}_k \in (-c, c)$. Notice that, $c$ here indicates the
	camera threshold.
	
	Next, we can describe a collection of $N_e \in \mathbb{N}$ events with the following definition: 
	\[
	\mathbb{R}^{N_e \times 4} \ni \bm{e} := (\bm{x},\bm{y},\bm{t},\bm{p}).
	\]
	We are interested in whether an event has occurred
		or not, the positive or negative sign on polarity is not relevant here. Therefore,
		we will omit the polarity column vector and redefine $\bm e$ as:
	\[
	\mathbb{R}^{N_e \times 3} \ni \bm{e} := (\bm{x},\bm{y},\bm{t}).
	\]
	A visual representation of event data is given in Figure \ref{f:ndata}.
	\begin{figure}[!htb]
		\centering
		\includegraphics[width=.45\textwidth]{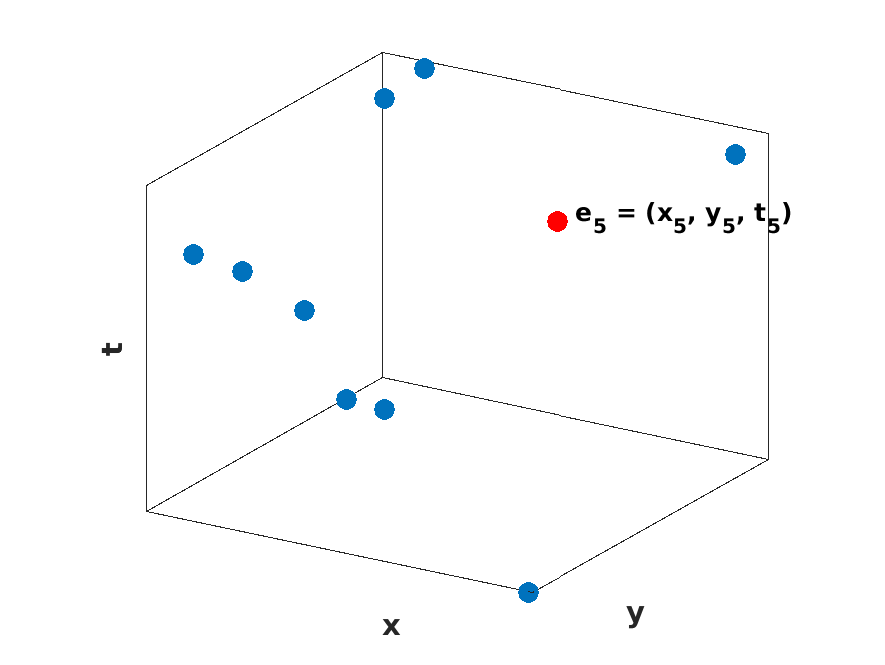}  \qquad
		\caption{In this figure we isolate an event, $\bm{e}_5$, shown as the red dot.}
		\label{f:ndata}
	\end{figure}
	\subsection{Warping Function}\label{s:warp}
	In the context of this paper the warping function refers to a mapping $\bm{W}: \mathbb{R}^3 \rightarrow \mathbb{R}^2$. In general, given some motion parameter $\bm{\theta}$ and a reference time $t_0 \in \mathbb{R}$ we project the events, $\bm{e}$ to this reference time. This process is shown visually in Figures \ref{f:c_projlines}, \ref{f:c_proj2} in Appendix~\ref{s:app}.
	
	\subsection{Row-wise 2-Norm}\label{s:r2n}
	We define the operator $\|\cdot \|_{2,r}: \mathbb{R}^{M \times N} \rightarrow \mathbb{R}^M$ to be a row-wise 2 norm. Given a matrix $V \in \mathbb{R}^{M \times N}$ we have the following:
	\[
	\|V\|_{2,r} = (\|V_{(1,:)}\|_2 \;\; \|V_{(2,:)}\|_2 \;\;...\;\; \|V_{(M,:)}\|_2 )^\top.
	\]
	Here $V_{(M,:)}$ represents the $M$-th row of $V$. 
	\section{Prior Work}\label{s:pw}
	The approach presented in this paper draws inspiration from the Contrast Maximization (CM) technique, which we now formally introduce. The CM method seeks to find a set of motion parameters $\{\bm{\theta}^1,\bm{\theta}^2, \dots, \bm{\theta}^N\}$ (one for each object in the scene) that maximizes the contrast in an image of  warped events. For a visual interpretation of $\bm\theta^j$, we refer to Appendix~\ref{s:app}. This image of warped events, denoted as $\bm{I} \in \mathbb{R}^{N_x \times N_y}$, is generated by projecting clusters of events onto a reference time, $t_0 \in \mathbb{R}$, via a warping function (cf.~Appendix~\ref{s:app}). When the parameter $\bm{\theta}^j$ accurately captures the motion of a specific object, the projected events associated with that object coalesce into a distinct and well-defined cluster within the image $\bm{I}$.  The level of clarity in this image can be assessed using various functions, such as Image Variance, Magnitude of the Image Gradient, and Image Entropy \cite{gallego2019focus}. The level of clarity serves as a metric for how accurately the motion parameter $\bm{\theta}$ approximates the true motion. Furthermore, given a set of motion parameters, $\{ \bm \theta^j \}_{j=1}^N$, we are then able to perform event based image segmentation by segmenting the events by their respective motion parameters. A general outline for techniques centering on contrast maximization (for a single object)  is given by \cite{zhou2021event}: 
	\begin{enumerate}
		\item Estimate/Initialize a motion parameter, $\bm{\theta}$. 
		\item Use a warping function, $\bm{W}(\bm{\theta})$, to project events to a reference time, $t_0$, which generates an image of warped events, $\bm{I}$.
		\item Evaluate $\bm{I}$ with a reward function.
		\item Perform an optimization step with respect to the motion parameter $\bm{\theta}.$
	\end{enumerate}
	Some care has to be given to initializing $\bm{\theta}$ since this problem is inherently non-convex and there are multiple local minima to consider \cite{stoffregen2019event}. Previous efforts have utilized prior knowledge such as optical flow information or other initialization techniques as shown in \cite{glover2016event,stoffregen2019event,vasco2017independent,zhou2021event}.
	In scenarios where multiple objects coexist within a scene several challenges arise: 
	\begin{enumerate}
		\item The difficulty in initialization of motion parameters is compounded due to multiple objects.
		\item When faced with events generated by multiple objects, the task extends beyond identifying motion parameters for each object. It becomes crucial to accurately discern the specific events associated with each object. This results in the need for a clustering step to be added to the contrast maximization outline. \cite{stoffregen2019event,zhou2021event}. 
		\item Additional prior knowledge may be needed such as number of objects in the scene \cite{stoffregen2019event}.
	\end{enumerate}
	
	Our work is a departure from the above standard approach. We do not use the motion parameter, $\bm{\theta}$, as our optimization variable. Instead we establish a GNEP framework where the motion parameters $\{\bm{\theta}^j\}_{j=1}^{N}$ are formulated as a function of the events and the strategies chosen by each player. We then seek to optimize the strategy of each player. With this framework we are able to avoid the need for prior knowledge in order to initialize the motion parameters and we do not need to perform an alternating optimization method as in \cite{stoffregen2019event} and \cite{zhou2021event}. In these prior methods, two variables had to be optimized, the motion parameters $\{\bm{\theta}^j\}_{j=1}^N$ and the cluster associations for each event, $\bm{P}_{j,k}$, with one variable being held fixed while the other is iteratively optimized.

	\section{A GNEP Based Model}\label{s:model}
	In our model, each player is assigned the task of selecting a strategy that maximizes their own local objective function, denoted as $J^j$. The chosen strategy involves assigning a confidence value to each event, where the confidence value assigned to event $k$ reflects the player's level of certainty that the events with non-zero confidence scores are generated by the same object.  An outline for each player is given as follows: 
	
	\begin{enumerate}
		\item Player $j$ chooses a strategy $\bm{z}^j$ while taking into account the strategies of the other players, $\bm{z}^{-j}$. We show an optimal strategy for two players in Figure \ref{f:c_p1p2}.
		\item Calculate the motion parameter for player $j$, denoted by $\bm{\theta}^j$, using a least squares approximation. This average motion of events is formally defined in Section \ref{s:ame}. The slope of the lines shown in Figure \ref{f:c_projlines} represent $\bm{\theta}^j$.
		\item Calculate the image of warped events, $\bm{I}^j$ based upon the strategy chosen, $\bm{z}^j$. An optimal image of warped events is shown in Figure \ref{f:c_proj2}.
		\item Evaluate the image of warped events, $\bm{I}^j$, with regard to a reward function.
		\item Perform an optimization step with respect to the strategy $\bm{z}^j$.
	\end{enumerate}
	Next, we describe our model mathematically.
	
	\subsection{Generalized Nash Equilibrium Model}
	Consider a set of ${N}$ players such that $j = 1,\dots,{N}$ and event data $\bm{e} \in \mathbb{R}^{N_e \times 3}$. 
	We will denote the $j$-th player strategy as:
	\begin{align}\label{eq:z}
	\bm{z}^j \in \bm{Z}^j (\bm{z}^{-j}) := \biggl\{\bm{z}^j \in \big[{0}, {1}\big]^{N_e} \; | \; \bm{g}(\bm{z}) \leq \bm{0}\biggr\}.
	\end{align}
	The strategy $\bm{z}^j$ represents the confidence values assigned by player $j$ to indicate their level of confidence 
	in the relationship between the events.
	Recall $N_e$ is the total number of events and $\bm{g}(\bm z)$ is a set of linear constraints given by: 
	\[
	\bm{g}(\bm z) := \sum_{j=1}^N \bm{z}^j - \bm{b}, 
	\]
	with\[ \bm{b} := \begin{pmatrix}
	1  &
	1 &
	\cdots &
	1
	\end{pmatrix}^\top \in \mathbb{R}^{\bm{N}_e}.\]
	The linear constraints, represented by the function $\bm{g}(\bm{z})$, impose a restriction on the sum of confidence values assigned to each event. This restriction ensures that the cumulative confidence values for each event do not exceed a maximum value of 1. By incorporating these constraints, we maintain the coherence and integrity of the confidence assignments, enabling a more accurate representation of the associations between events. 
	Next we will begin building our objective function starting with calculating $\bm{\theta}.$ 
	
	\subsection{Average Motion of Events}\label{s:ame}
	The average motion of events is the least squares approximation of the change in $x,y$ directions with respect to time of the chosen strategy.  We denote $\mathbb{H}(s)$ to be the Heaviside step function: 
	
	\begin{align}\label{eq:H}
	\mathbb{H}(s) := \begin{cases}
	1, \qquad s > 0 \\
	0,\qquad s \le 0 \, .
	\end{cases} 
	\end{align}
	
	We then calculate the change in the $x$-direction with respect to time of the events chosen in strategy 
	$\bm{z}^j$ by the following:
	\begin{align}\label{eq:tx}
	\theta_x^j :=   \frac{\mathbb{H}\big(\bm{z}^{j}\big)^\top \big((\bm{t} - \mu_{t}(\bm z^j) \cdot \bm{b}) \odot (\bm{x} - \mu_x(\bm z^j) \cdot \bm{b})\big)}{\mathbb{H}\big(\bm{z}^{j}\big)^\top \big((\bm{t} - \mu_{t}(\bm z^j) \cdot \bm{b}) \odot (\bm{t} - \mu_t(\bm z^j) \cdot \bm{b})\big)}.
	\end{align}
	
	Notice that $\mathbb{H}$ acts component-wise when applied to vectors.
	Furthermore $\mu_t(\bm{z}^j)$ and $\mu_x(\bm{z}^j)$ are defined as
	\begin{align}\label{eq:mx}
	\mu_t(\bm{z}^j) := \frac{\mathbb{H}\big(\bm{z}^{j}\big)^\top\bm{t}}{\sum_{k=1}^{N_e}{\mathbb{H}((\bm{z}^j)_k)}}
	\quad \mbox{and}
	\quad 
	\mu_x(\bm{z}^j) := \frac{\mathbb{H}\big(\bm{z}^{j}\big)^\top\bm{x}}{\sum_{k=1}^{N_e}{\mathbb{H}((\bm{z}^j)_k)}}.
	\end{align}
	We calculate $\theta_y^j$ similarly using $
	\bm{y}$ in place of $\bm{x}$ in Equation \eqref{eq:tx}. This gives us:
	\[
	\bm{\theta}^j := (\theta_x^j,\theta_y^j).
	\]
	\subsection{Warping Function}
	We define a velocity based warping function as follows: 
	\begin{equation}\label{eq:W}
	\mathbb{R}^{N_e \times 2} \ni \bm{W}^j(\bm{e},t_0,\bm{z}^j):= (\bm{x},\bm{y}) + ( t_0\cdot \bm{b}-\bm{t})\odot \big({\theta}^j_x \cdot \bm{b}, \; \theta_y^j \cdot \bm{b} \big) .
	\end{equation}
	Here the symbol $\odot$ represents element-wise multiplication known as the 
	Hadamard product. Next we formally define an image of warped events for 
	each player $j$.

	\subsection{Image of Warped Events}
	The image of warped events is the result of translating the events chosen by strategy $\bm{z}^j$ to a given time, $t_0 \in \mathbb{R}$ with the warping function $\bm{W}^j(\bm{e},t_0,\bm{z}^j).$ 
	We describe this process mathematically for a given pixel, $(\hat{x},\hat{y}) \in \bm{\Omega}$, as follows:
	\begin{align}\label{eq:I}
	\mathbb{R} \ni \bm{I}^j((\hat{{x}},\hat{{y}}),t_0,\bm{z}^j) := \big(\bm{z}^j\big)^\top  \delta\bigg(\| \bm{W}^j(\bm{e},t_0,\bm{z}^j) - (\hat{{x}} \cdot \bm{b}, \; \hat{{y}} \cdot \bm{b})\|_{2,r}\bigg).
	\end{align}
	Here $\| \cdot \|_{2,r}$ is the row-wise 2 norm defined in Section \ref{s:r2n} and $\delta$ 
	refers to the Kronecker delta centered at 0 defined as:
	\begin{align}\label{eq:kron}
	\delta(\gamma) = \begin{cases}
	0, \qquad \gamma \neq 0 \\
	1,\qquad \gamma = 0;
	\end{cases}
	\end{align}
	and is applied elementwise to vectors.
	We describe the full image of warped events as: 
	\[
	\bm{I}^j({\bm{\Omega}},t_0,\bm{z}^j) \in \mathbb{R}^{N_x \times N_y}.
	\]
	For the remainder of the paper we will omit $t_0$ from $\bm{I}^j$ and $\bm{W}^j$ for ease of notation. 
	\subsection{Objective Function}\label{s:objfun}
	Finally, we define our objective function $J^j$ as: 
	\begin{align}\label{eq:J}
	J^j(\bm{z}^j):= -E\big(\bm{I}^j({\bm{\Omega}},\bm{z}^j)\big) + \frac{\lambda_1}{2}\bigg(V(\bm{z}^j,\bm{x}) + V(\bm{z}^j,\bm{y})\bigg) + \frac{\lambda_2}{2} \|\mathbb{H}(\bm{z}^j)\|^2_2.
	\end{align}
	Where $E(\bm{I}^j)$ is a modified entropy function given by: 
	\begin{align}\label{eq:E}
	E(\bm{I}^j) = \sum_{\hat{{x}}=1}^{N_x}\sum_{\hat{{y}}=1}^{N_y}\bm{I}^j((\hat{{x}},\hat{{y}}),\bm{z}^j) \cdot \log\bigg(  \bm{I}^j((\hat{{x}},\hat{{y}}),\bm{z}^j)\bigg),
	\end{align}
	and
	\begin{align}\label{eq:Vj}
	V(\bm{z}^j,\bm{x}):= \frac{1}{N_x} \Bigg( \bigg(\mathbb{H}\big(\bm{z}^{j}\big) \odot \bm{x}\bigg) - \mu_{x}(\bm{z}^j)\cdot \bm{b} \Bigg)^\top\Bigg(\bigg(\mathbb{H}\big(\bm{z}^{j}\big) \odot \bm{x}\bigg) - \mu_{x}(\bm{z}^j)\cdot \bm{b} \Bigg).
	\end{align}
	Similarly, we define $V(\bm{z}^j,\bm{y})$. Notice that $V(\bm{z}^j,\cdot)$ denotes the 
	variance penalty term with $\lambda_1 \geq 0$ being the penalty parameter. Furthermore we have 
	the regularization term $\frac{\lambda_2}{2}\|\mathbb{H}(\bm{z}^j)\|^2_2$ with $\lambda_2 \ge 0.$ 
	This term serves to limit the number of extraneous events (noise) chosen by each player.
	Then we have the following model for each player, $j = 1,\dots,N$:
	\begin{align}\label{eq:Model}
	&\min \; \; \; J^j(\bm{z}^j) \; , \\
	&\mbox{subject to} \; \; \; \; {\bm{z}^j \in\bm{Z}^j(\bm{\bm{z}}^{-j})}, \nonumber
	\end{align}
	where $\bm{Z}^j(\bm{z}^{-j})$ is defined in \eqref{eq:z}.
	
	Since the objective functions, $J^j(\bm{z}^j)$, are not convex we are not able to solve \eqref{eq:Model} 
	directly using the standard approach of quasi-variational inequalities. Instead we rely on 
	leader / follower methods in the form of a $N$-level optimization problem to search for equilibrium of 
	\eqref{eq:Model}, where $N$ is the number of players \cite[C.3, S.4, pp.68-73]{dempe2020bilevel}. We present the $N$-level optimization problem as:
	\begin{align}\label{eq:opt}
	\min_{\bm{z}^N \in\bm{Z}^N(\bm{\bm{z}}^{-N})} J^N(\bm{z}^N)
	\end{align}
	subject to 
	\begin{align*}
	\bm{z}^j = \argmin \Big\{J^j(\bm{z}^j) \; | \; \sum_{i=1}^{j} \bm{z}^i\leq \bm{b}, \, \bm{z}^j \in [0,1]^{N_e} \Big\} \quad \text{for } j= 1,\dots,N-1.
	\end{align*}
	We notice a subtle difference between $\eqref{eq:opt}$ and $\eqref{eq:Model}$. Each player 
	in $\eqref{eq:Model}$ has complete information of all other player strategies. While in $\eqref{eq:opt}$ 
	each player only takes the strategy of the previous players into account. For example, for player 2 we have the following problem:  
	\[\bm{z}^2 = \argmin \Big\{J^2(\bm{z}^2) \; | \; \sum_{i=1}^{2} \bm{z}^i\leq \bm{b}, \, \bm{z}^2 \in [0,1]^{N_e} \Big\}.\] 
	Expanding the constraint we get:
	\[ 
	\sum_{i=1}^2 \bm{z}^i= \bm{z}^1 + \bm{z}^2.
	\] 
	Hence we must first calculate $\bm{z}^1$ before calculating $J^2$. Notice that, the player $N$ accounts for the strategies of players $1,\dots,N-1$. 
	For this reason, problems from $\eqref{eq:opt}$ will need to be solved in ascending order starting with player 1. 
	We note that we solve this problem in a single pass. 
\begin{algorithm}[H]
	\caption{N-level Optimization Algorithm}
	\flushleft
	\textbf{Input:} $\bm{e}$ - Event data, $\lambda_1,\lambda_2$ - Penalty parameters
	\begin{algorithmic}[1]
		\setstretch{1.3} 
		\For{$j = 1$ to $N$}
		\State Initialize $\bm{z}^j_{init}$ 
		\State Calculate $\bm{\bar{z}}^j$ shown in Equation~\eqref{eq:opt} \Comment $\bm{z}^j_{init}$ is used as the initial value.
		\EndFor
		\State \textbf{Return} $\bm{\bar{z}} = \begin{pmatrix}
		\bm{\bar{z}}^1 & \bm{\bar{z}}^2 &\cdots& \bm{\bar{z}}^N
		\end{pmatrix}^\top$ \Comment{This is the approximation to solution in Def.~\ref{d:gne}.}
	\end{algorithmic}
\end{algorithm}

	In the next section we show the existence of solutions for \eqref{eq:opt}, and that the set of 
	solutions for $j=1,\dots,N$ corresponds to an equilibrium of \eqref{eq:Model}.
	\section{Analysis of the Proposed Model}\label{s:analysis}
	
	\subsection{Existence of Equilibrium}
	To establish the validity of the $N$-level optimization presented in Equation \eqref{eq:opt} as a method 
	for computing an equilibrium to \eqref{eq:Model}, we begin by demonstrating the existence of a solution. 
	This process involves several steps, with the initial focus on establishing the finiteness of the set 
	$\{\mathbb{H}(\bm{z^j}) \; | \; \bm{z}^j \in [0,1]^{N_e}\}$.
	
	\begin{lemma}\label{l:finite}
		The set \{$\mathbb{H}(\bm{z}^j) \; | \; \bm{z}^j \in [0,1]^{N_e}\}$ is finite. 
	\end{lemma}
	\begin{proof}
		For ease of notation we omit the superscript $j$ from $\bm z$. 
		We notice that $\mathbb{H}(\bm{z}) \in \{0,1\}^{N_e}$ for all $\bm{z}$, thus $\{\mathbb{H}(\bm{z}) \; | \; \bm{z} \in [0,1]^{N_e}\} \subseteq \{0,1\}^{N_e}.$ Furthermore the set $\{0,1\}^{N_e}$ has cardinality of $2^{N_e}$ and is therefore finite. 
		Hence $\{\mathbb{H}(\bm{z}) \; | \; \bm{z} \in [0,1]^{N_e}\}$ is finite.
	\end{proof}
Next we show that the quantities $V(\bm{z},\cdot)$ and $V(\mathbb{H}(\bm{z}),\cdot)$ are equal.
\begin{lemma}\label{l:veq}
		For any $\bm{z} \in \mathbb{R}^{N_e}$ we have $V(\bm{z},\cdot) =V(\mathbb{H}(\bm{z}),\cdot) $.
\end{lemma}
\begin{proof}
	This follows directly from the definition of $V$ in \eqref{eq:Vj}.
\end{proof}
We are now ready to establish existence of solution to \eqref{eq:opt}.
\begin{theorem}\label{thm:exist}
	There exists a solution to each level of the optimization problem \eqref{eq:opt}.
\end{theorem} 
\begin{proof}
	For ease of notation we omit the superscript $j$. 
	
	\smallskip
	\noindent	
	\boxed{\text{Part I: Single player case.}} 
	We focus on player 1 as the result will directly
	apply to other players. First we observe for any $\bm{z} \in [0,1]^{N_e}$  \[\mathbb{H}(\bm{z}) \leq \bm{b}\] 
	by definition of the Heaviside function given in Equation~\eqref{eq:H}.
	Next we note from Lemma~\ref{l:finite} that the set $\{\mathbb{H}(\bm{z}) \; | \; \bm{z} \in [0,1]^{N_e}\}$ is finite. 
	Since the set is finite we know there exists some $\hat{\bm z} \in [0,1]^{N_e}$ such that 
	 $\bar{z} := \mathbb{H}(\hat{\bm{z}})$ fulfills
	\begin{align}\label{eq:minJ}
		J(\bar{\bm{z}}) \leq J(\mathbb{H}(\bm{z})) \quad \text{ for all } \quad \bm{z} \in [0,1]^{N_e}.
	\end{align}
	Here $\bar{\bm z}$ is our candidate for a solution to \eqref{eq:opt}. 
	Next we show for arbitrary $\bm{z}$ that $J(\mathbb{H}(\bm{z})) \leq J(\bm{z})$.
	If $\bm{z} = \mathbb{H}(\bm{z})$ then we are done. 
	Else, there exist some $k = 1,\dots,N_e$ such that $ 0 < \bm{z}_k < 1$.
	Then for any ($\hat{x}$, $\hat{y}) \in \bm{\Omega}$ and setting
	\[
		\boxed{\textrm{I}} := \| \bm{W}(\bm{e},\bm{z}) - (\hat{{x}} \cdot \bm{b}, \; \hat{{y}} \cdot \bm{b})\|_{2,r}
	\]
	we immediately obtain that 	
	\begin{align}\label{eq:zwhzw}
	\bm{z}_k \cdot \delta\left(\boxed{\textrm{I}}\right)_k 
	\leq \mathbb{H}(\bm{z}_k) \cdot \delta\left(\boxed{\textrm{I}}\right)_k ,
	\end{align}
	where $\delta\left(\boxed{\textrm{I}}\right)_k$ 
	represents the $k$-th element of 
	$ \delta\left(\boxed{\textrm{I}}\right). $ 
	Using Equation~\eqref{eq:zwhzw} and the fact that log is a monotone increasing function we have the 
	following inequality:
	\begin{align}\label{eq:IzIHz}
	 \bm{z}_k \cdot \delta\left(\boxed{\textrm{I}}\right)_k  \cdot \log\left(\bm{z}_k \cdot \delta\left(\boxed{\textrm{I}}\right)_k\right) 
	  \leq  \mathbb{H}(\bm{z}_k) \cdot \delta\left(\boxed{\textrm{I}}\right)_k \cdot \log\left(\mathbb{H}(\bm{z}_k) \cdot \delta\left(\boxed{\textrm{I}}\right)_k\right).
	\end{align} 
	We note the common definition of $0 \cdot \log(0) := 0$, hence this holds even for 
	$\delta\left(\boxed{\textrm{I}}\right)_k = 0.$ 
	Using the definition of $\bm{I}(\bm{\Omega}, \cdot)$ in \eqref{eq:I} and Equation~\eqref{eq:IzIHz} we have:
	\[\bm{I}((\hat{x},\hat{y}),\bm{z}) \leq \bm{I}((\hat{x},\hat{y}),\mathbb{H}(\bm{z})).\]
	Taking the sum over all possible $(\hat{x},\hat{y})\in \bm{\Omega}$ gives us:
	\begin{align}\label{eq:EIEHI}
	E\big(\bm{I}(\bm{\Omega},\bm{z})\big) \leq E\big(\bm{I}(\bm{\Omega},\mathbb{H}(\bm{z}))\big).
	\end{align}
	Combining Equations~\eqref{eq:minJ} and \eqref{eq:EIEHI} along with Lemma~\ref{l:veq} we have for any $\bm{z} \in [0,1]^{N_e}$:
	\[J(\bm{z}) \geq J(\mathbb{H}(\bm{z})) \geq J(\bm{\bar{z}}).\]
	We conclude that a solution to \eqref{eq:opt} exists for a single player. 
	
	\smallskip
	\noindent	
	\boxed{\text{Part II: Two player case.}} 
	We now show that given  
	\[\bar{\bm{{z}}}^1 
		= \argmin \Big\{J^1(\bm{z}^1) \; | \; \bm{z}^1\leq \bm{b}, \, \bm{z}^1 \in [0,1]^{N_e} \Big\},
	\]  
	a solution exists for 
	\[
		\min \Big\{J^2(\bm{z}^2) \; | \; \sum_{i=1}^{2} \bm{z}^i\leq \bm{b}, \, 
		\bm{z}^2 \in [0,1]^{N_e} \Big\}.
	\]
	First we notice from \boxed{\textrm{Part I}} that the components of $\bar{\bm z}^1$ are either 0 
	or 1. Now, define the set 
	\[
		M = \{ m = 1,\dots,N_e \; | \; (\bar{\bm z}^1)_m = 1 \} \subset \mathbb{R}^{\#M} \, .
	\]
	Then for $m \in M^c$ we have $(\bar{\bm z}^1)_m = \bm 0 \in \mathbb{R}^{\#M^c}$. Since we have 
	\[
		\bm b \ge \sum_{i=1}^2 \bm z^i  = 
		\begin{pmatrix}
			(\bar{\bm z}^1)_{m\in M} + (\bar{\bm z}^2)_{m \in M}  \\
			(\bar{\bm z}^1)_{m\in M^c} + (\bar{\bm z}^2)_{m \in M^c}  
		\end{pmatrix}.
	\]
	Using the definition of $\bm b$, for $m \in M$, we have that $(\bar{\bm z}^2)_{m} = \bm 0 \in \mathbb{R}^{\#M}$ 
	and for $m \in M^c$ we have that $(\bar{\bm z}^2)_{m} \in [0,1]^{\#M^c}$. This proves the existence of
	$\bar{\bm z}^2$. Similar arguments helps prove existence of $\bar{\bm z}^3$, etc. and the proof is complete.	
\end{proof}

Next, we show that the $N$-level optimization presented in Equation \eqref{eq:opt} produces an equilibrium point to \eqref{eq:Model}.
\begin{theorem}
	The N-level optimization described in \eqref{eq:opt} produces a Generalized Nash Equilibrium to \eqref{eq:Model}.
\end{theorem}
\begin{proof}
	Let $\bar{\bm{z}}^j$, $j=1,\dots,N$, be the resulting argmin from each level of \eqref{eq:opt}. 
	Then we have the following:
	\begin{align*}
		&\sum_{j=1}^N \bar{\bm{z}}^j \leq \bm{b}
		\implies \sum_{j=1}^N \bar{\bm{z}}^j - \bm{b} \leq \bm{0}.
	\end{align*}
	Furthermore, from $\eqref{eq:opt}$ we know $\bar{\bm{z}}^j \in [0,1]^{N_e}.$ Then for each $j$ we have
	 \[\bar{\bm{z}}^j \in  \biggl\{[0,1]^{N_e} \; | \; \sum_{i=1}^N \bar{\bm{z}}^i - \bm{b} \leq \bm{0} \biggr\}   =: \bm{Z}^j(\bar{\bm{z}}^{-j}).\]
	We conclude by Definition~\ref{d:gne}
	$ 
	\bar{\bm{z}} = 
	\begin{pmatrix}
	\bar{\bm{z}}^1 & \bar{\bm{z}}^2 & \cdots & \bar{\bm{z}}^N
	\end{pmatrix}^\top\,
	$ is an equilibrium point of \eqref{eq:Model}.
\end{proof}
%

\section{Experimental Setup}\label{s:numres}

\subsection{Event Data Pre-processing}
We observe that the dimension of $\bm{z}^j$ directly corresponds to the number of events under consideration. As the number of events increases, so does the dimensionality of the variable, potentially leading to computational challenges. To mitigate this issue, we employ a pre-processing step where we selectively remove time intervals of events from our data.
By doing so, we effectively limit the dimensionality of the optimization variable, allowing for more efficient and manageable computations while still preserving the essential information contained within the dataset. Figure \ref{f:data} gives a visual example.

\begin{figure}[!htb]
	\centering
	\includegraphics[width=.45\textwidth]{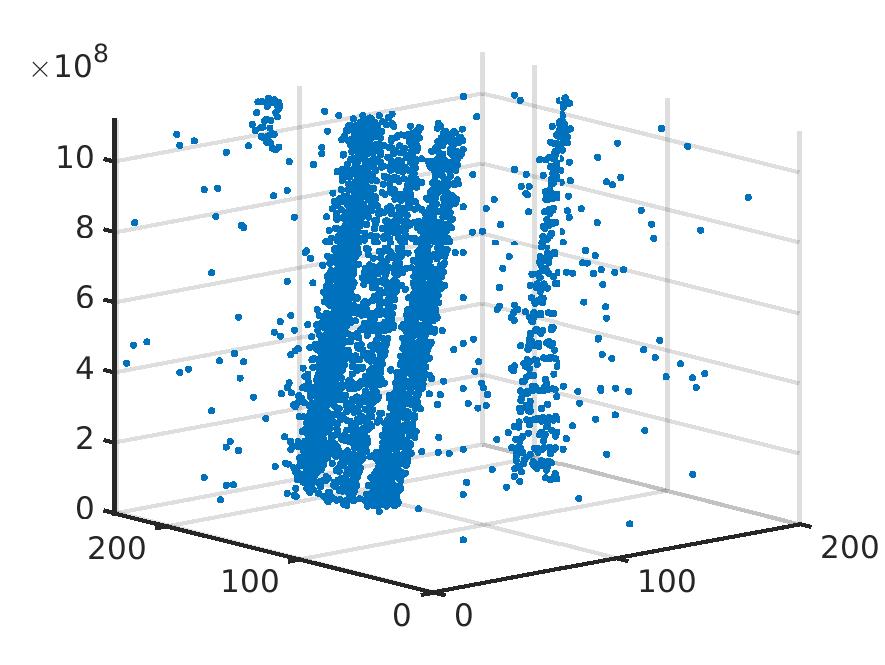}  \qquad
	\includegraphics[width=.45\textwidth]{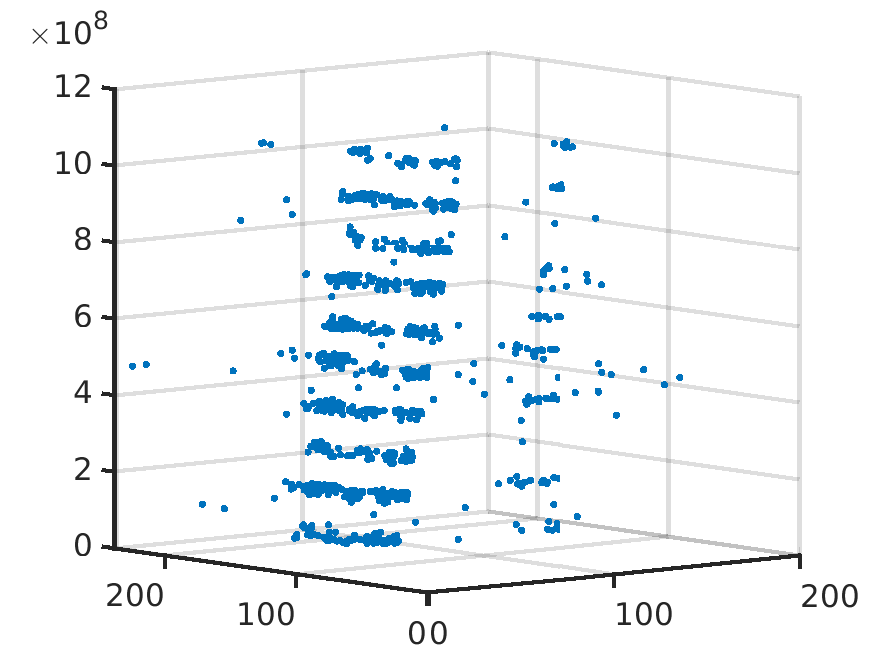} \qquad
	\caption{In this figure we see the original event data in the left panel and the pre-processed data in the right panel. }
	\label{f:data}
\end{figure}

\subsection{Numerical Relaxations}\label{s:nr}
We make several numerical relaxations in order to allow for gradient based methods.

\subsubsection{Heaviside Relaxation}
First we relax the Heaviside function described in \eqref{eq:H} to a continuous approximation given 
by: 
\[
	\mathbb{\hat{H}}(\bm{z}^j) := \frac{1}{2}\big(1 + \tanh(m \cdot (\bm{z}^j - n \cdot \bm{b}))\big).
\]
 Figure~\ref{f:H1} gives a visual representation of the approximation. For the purposes of this paper we chose $m = 25$ and $n=0.25$ for each of our examples. Through testing we found for $m > 50$ that the gradient of 
 $\mathbb{\hat{H}}(\bm{z}^j)$ becomes unstable. We see this visually in Figure~\ref{f:rh1}. Furthermore, we choose $n>0$ to ensure $\mathbb{\hat{H}}(0) \approx 0$ and $\mathbb{\hat{H}}(1) \approx 1$. If this is not the case, then the quality of the results decreases significantly.
 \begin{figure}[!htb]
 	\centering
 	\includegraphics[width=.45\textwidth]{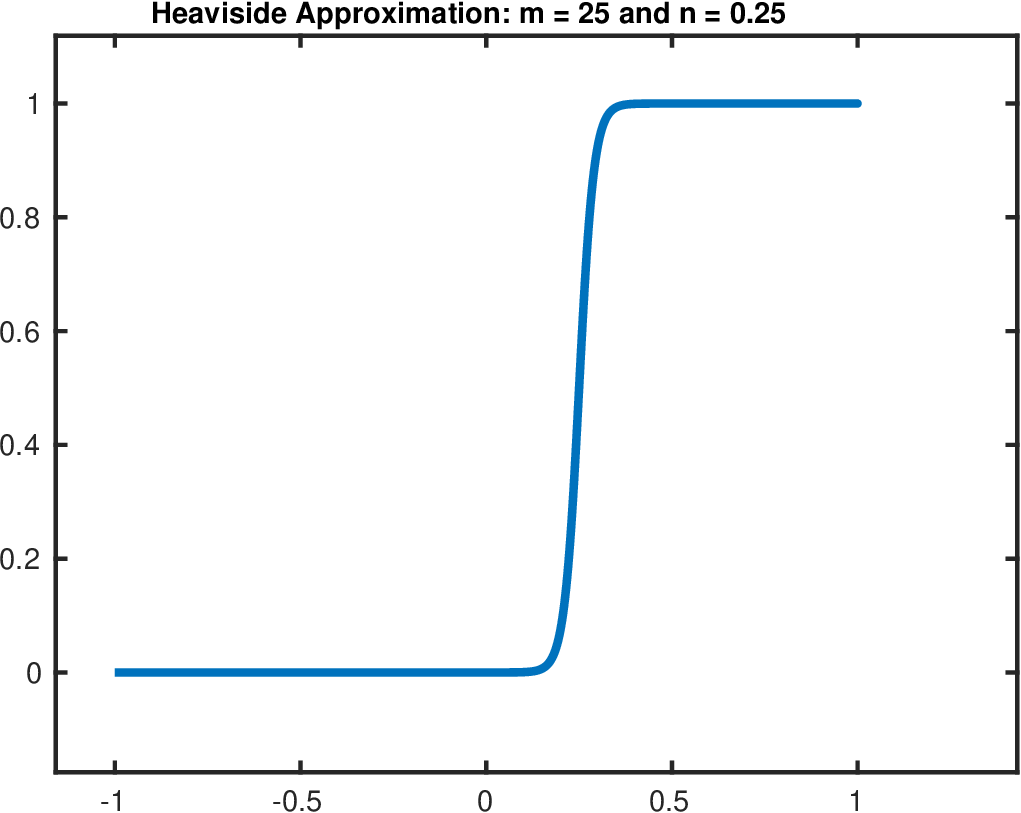}  \qquad
 	\caption{In this figure give a visual representation of the Heaviside relaxation function $\mathbb{\hat{H}}(\bm{z})$ (for $\bm{z} \in [-1,1])$ }
 	\label{f:H1}
 \end{figure}
 \DS{\begin{figure}[!htb]
 	\centering
 	\includegraphics[width=.3\textwidth]{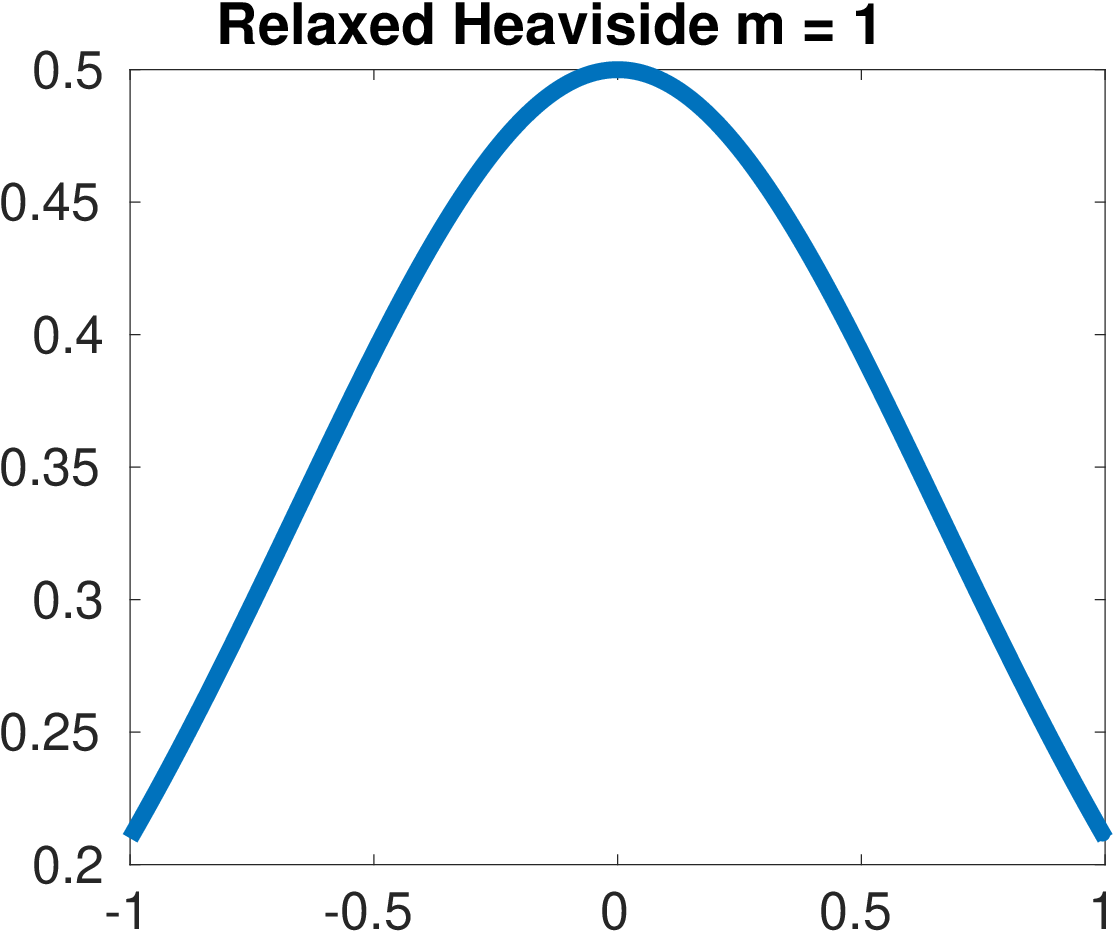}  \qquad
 	\includegraphics[width=.29\textwidth]{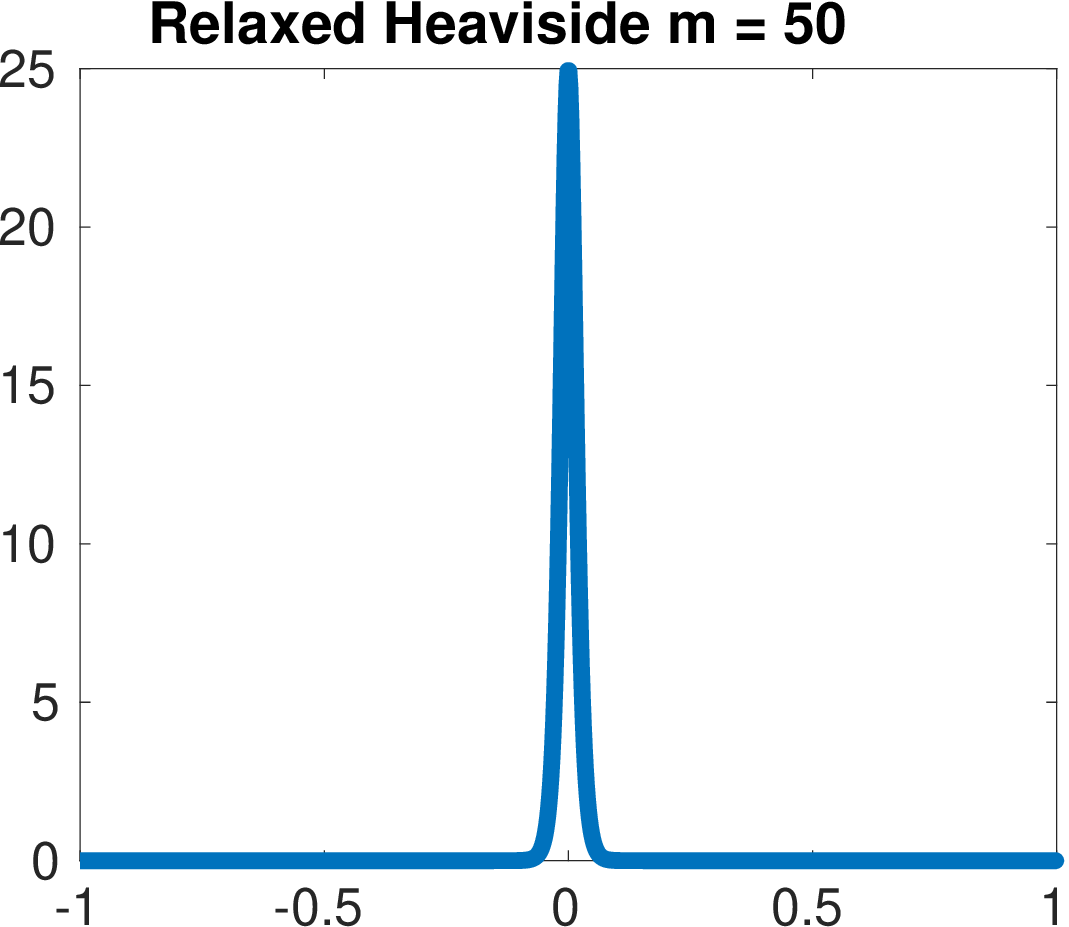} \qquad
 	\includegraphics[width=.29\textwidth]{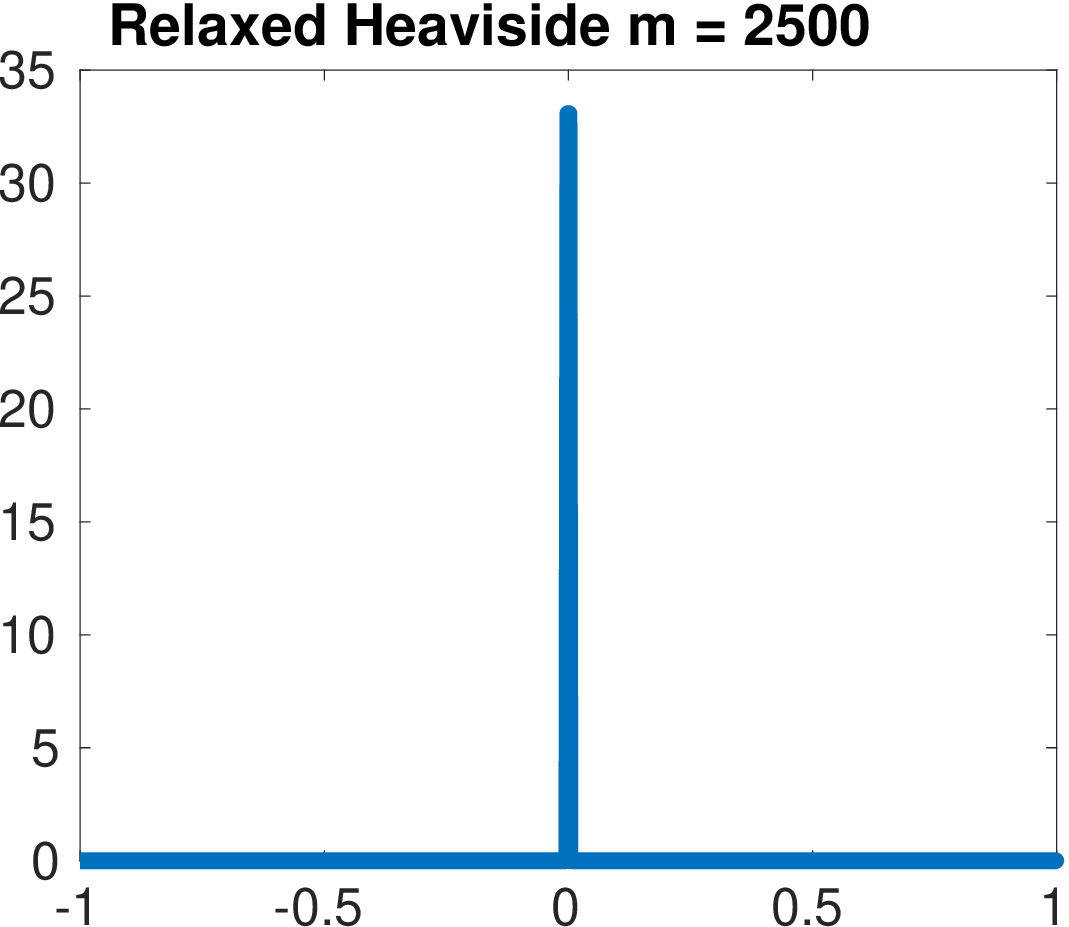}
 	\caption{In this figure we plot $\nabla\mathbb{\hat{H}}(\bm{z}^j)$ for $m = 1$, $50$, and $2500$. 
	As $m$ increases, $\nabla\mathbb{\hat{H}}(\bm{z}^j)$ approximates an impulse which can be 
	unstable in numerical calculations.}
 	\label{f:rh1}
 \end{figure}}
\subsubsection{Kronecker Delta Relaxation}
We also use the function: 
\[
	\exp\big(- \gamma \| \cdot \|^2_2\big)
\]
to approximate the Kronecker delta $\delta \big(\| \cdot \|_2\big)$ defined in \eqref{eq:kron}. 
Here $\gamma$ is chosen to be large enough such that the bump function decays rapidly. 
For the purposes of this paper we choose $\gamma = 30$. For $\gamma < 5$ there can be 
a drop off in the quality of results due to the bump function significantly overlapping more 
than one pixel. This can be seen  numerically in the following example: 
$\exp(-2\cdot\| (1,1) - (1,2) \|^2_2) = 0.1353$. We notice that (1,1) and (1,2) are one pixel 
apart, yet because $\gamma$ is small $0 \ll \exp(-2\cdot\| (1,1) - (1,2) \|^2_2) $, i.e., we 
have a poor approximation of Kronecker delta.  Figure~\ref{f:Kronapprox1} gives a visual 
representation of the relaxation function with $\gamma = 30$.

 \begin{figure}[!htb]
	\centering
	\includegraphics[width=.45\textwidth]{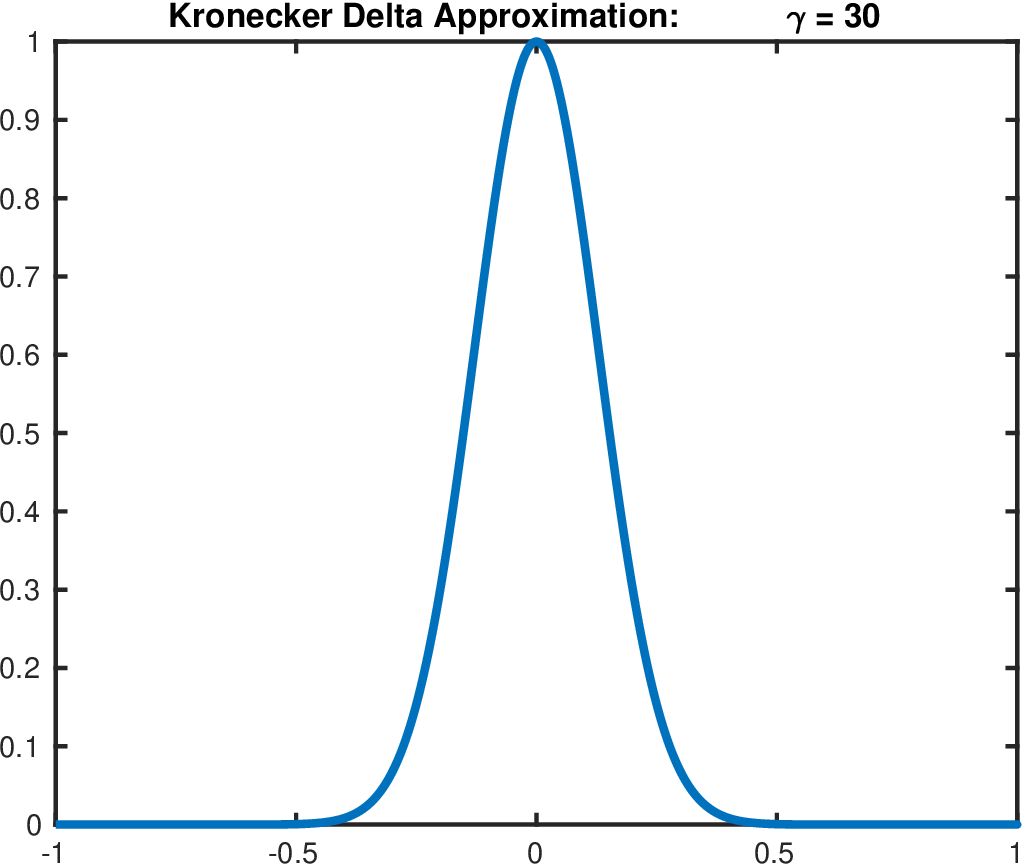}  \qquad
	\caption{In this figure give a visual representation of the Kronecker Delta relaxation function $ \exp(-\gamma \cdot \|\bm{z}\|^2_2)$ (for $\bm{z} \in [-1,1])$ }
	\label{f:Kronapprox1}
\end{figure}

\subsubsection{Entropy Relaxation}\label{Er}
We relax the Entropy function $E$ defined in \eqref{eq:E} by introducing two scalars 
parameters $\alpha$ and $\beta$:
\begin{align*}
E(\bm{I}^j) = \sum_{\hat{{x}}=1}^{N_x}\sum_{\hat{{y}}=1}^{N_y}\bm{I}^j((\hat{{x}},\hat{{y}}),\bm{z}^j) \cdot \log\bigg( \big(\alpha \big) \cdot \bm{I}^j((\hat{{x}},\hat{{y}}),\bm{z}^j) + \beta \bigg).
\end{align*}
These parameters are introduced in order to avoid $\log(0)$ numerically. We must impose the constraint: $\alpha = 1-\beta$ to ensure $\log(\alpha + \beta) = 0$. 
Under this condition the theoretical proofs in section~\ref{s:analysis} hold. We only check one of the
key inequalities \eqref{eq:IzIHz} next.
\begin{lemma}\label{l:ab}
	If $\alpha + \beta = 1$ where $\alpha, \beta \in (0,1]$ then the inequality \eqref{eq:IzIHz} becomes
	\[z \cdot \delta(\|\cdot\|^2) \cdot \log\bigg((\alpha) \cdot z \cdot \delta(\|\cdot\|^2)+ \beta\bigg) \leq \mathbb{H}(z) \cdot \delta(\|\cdot\|^2) \cdot \log\bigg((\alpha) \cdot\mathbb{H}(z) \cdot \delta(\|\cdot\|^2) + \beta\bigg)\]
	for any $\bm{z} \in [0,1]$.
\end{lemma}
\begin{proof}
Let $\alpha + \beta = 1$. Then for $z = 0$ or $z=1$ we have $z = \mathbb{H}(z)$ and we are done. Furthermore if $ \delta(\|\cdot\|^2) = 0$ we get equality and we are done. 
Now, suppose $ 0 < z < 1$ and $\delta(\|\cdot\|^2) =1$. Then we have the following: 
\begin{align*}
	z \cdot \delta(\|\cdot\|^2) \cdot \log\bigg((\alpha) \cdot z \cdot \delta(\|\cdot\|^2)+ \beta\bigg) 
	&= z \cdot \log\bigg((\alpha) \cdot z + \beta\bigg) < \log(\alpha + \beta) = 0.
\end{align*}
Furthermore we have:
\[
\mathbb{H}(z) \cdot \delta(\|\cdot\|^2) \cdot \log\bigg((\alpha) \cdot\mathbb{H}(z) \cdot \delta(\|\cdot\|^2) + \beta\bigg) = 0.
\]
We conclude 
\[
z \cdot \delta(\|\cdot\|^2) \cdot \log\bigg((\alpha) \cdot z \cdot \delta(\|\cdot\|^2)+ \beta\bigg) \leq \mathbb{H}(z) \cdot \delta(\|\cdot\|^2) \cdot \log\bigg((\alpha) \cdot\mathbb{H}(z) \cdot \delta(\|\cdot\|^2) + \beta\bigg),
\]
given $\alpha + \beta =1$.
\end{proof}
\subsection{Solvers: Genetic Algorithm and SQP Algorithm}\label{s:gasqp}

\subsubsection{Initialization Using Genetic Algorithm}
For our examples, we first use the MATLAB genetic algorithm, ga, to calculate an initial value for use in a 
gradient based method. Recall in Theorem~\ref{thm:exist} we showed that $\min J^j(\bm{z}^j)$ will occur at a boundary point, $\mathbb{H}(\bm{z}).$ For this reason, we give the Genetic Algorithm integer 
constraints of $0$ and $1$. 
Since the Genetic Algorithm is not a gradient method we do not use the function relaxations given in Section~\ref{s:nr}.

\subsubsection{SQP Algorithm}
Once the Genetic Algorithm has completed we then use the output as input to the MATLAB `sqp' 
algorithm which is part of fmincon. Since this is a gradient based method the relaxations shown 
in Section~\ref{s:nr} are used.

\section{Numerical Examples}
\label{s:nex}
Several examples of our method are shown. For each of these results, a combination of a Genetic Algorithm and SQP Algorithm as described in section~\ref{s:gasqp} were used. The tolerances for these methods were $10^{-2}$ and $10^{-4}$, respectively.

In Figure \ref{f:ex1}, 2016 events are used. The original (unprocessed) event data is shown in the left pane while the image of warped events for both players is shown in the right pane. As shown our method is able to fully segment the image into two distinct objects with player 1 associating events generated by the vehicle and player 2 associating events to the pedestrian. These results for player 1 required 447 iterations of a genetic algorithm and an additional 14 iterations from the SQP algorithm. In addition, the results for player 2 required 366 genetic algorithm iterations and 27 SQP iterations. Furthermore, for player 1, Figure~\ref{f:ex1c}, compares the output of the genetic algorithm to the output of the SQP algorithm. As shown in Figure~\ref{f:ex1c} the genetic algorithm does associate like events. The drawback being that there are events that should be associated that are left off. In contrast the SQP algorithm associates many more like events. Moreover, for this particular example the tolerance for each algorithm has been set to $10^{-4}$ for a fair comparison.

\begin{figure}[!htb]
	\centering
	\includegraphics[width=.45\textwidth]{pedestrian_car_orig.eps}  \qquad
	\includegraphics[width=.45\textwidth]{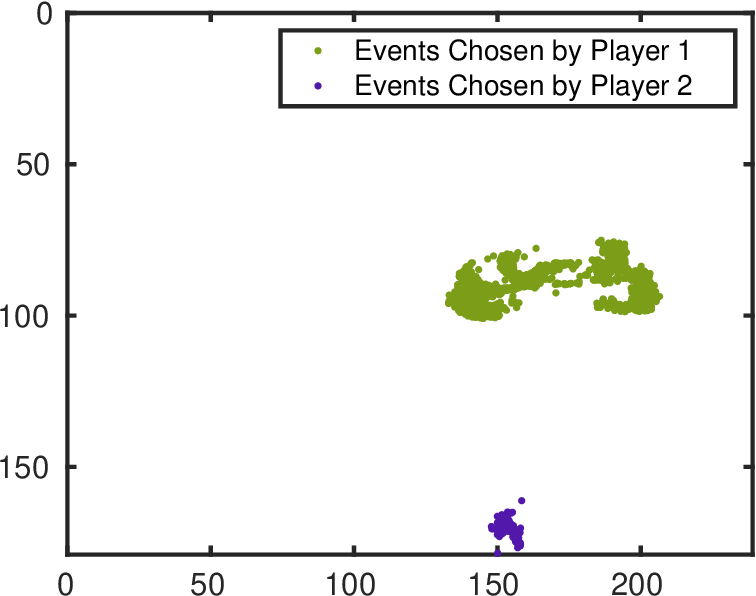}
	\caption{In this figure we show the original data in the leftmost pane. The results of our method with two players is shown in the right pane. In this example $N_e$, the number of total events, is 2016, $\lambda_1 = 10^{-3}$, $\lambda_2 = 1$, $\beta = 0.1$ and $\alpha = 0.9$. }
	\label{f:ex1}
\end{figure}

\begin{figure}[!htb]
	\centering
	\includegraphics[width=.45\textwidth]{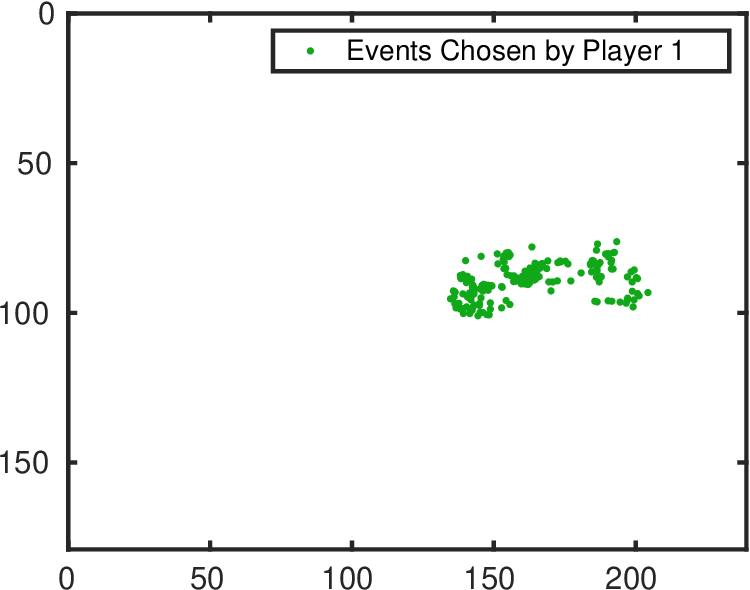}  \qquad
	\includegraphics[width=.45\textwidth]{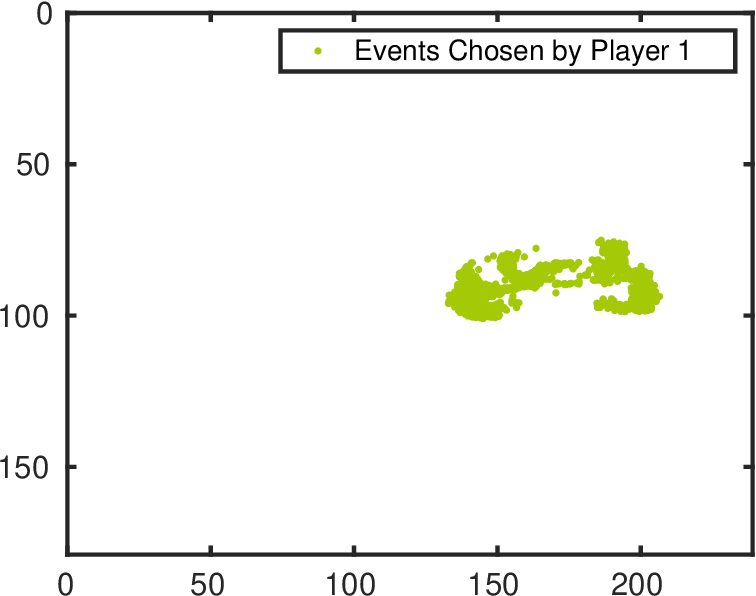}
	\caption{In this figure we show the the result of the genetic algorithm in the leftmost pane. In the right pane we show the result of the SQP algorithm, using the strategy produced by the genetic algorithm as an initial value. Notice that no relaxation of Heaviside function or Kronecker delta is used in case
	of the genetic algorithm. This example motivates that the results with relaxation are meaningful.}
	\label{f:ex1c}
\end{figure}

Figures \ref{f:ex2} - \ref{f:ex3} show our method applied to a multi-vehicle scenario. In Figure \ref{f:ex2}, 1490 events were used in the model. As shown (right) our model is able to segment the image into two distinct vehicles. For player 1 these results required 211 iterations of the genetic algorithm and an additional 8 iterations from the SQP algorithm. For the second player, 97 iterations of the genetic algorithm and 4 iterations of the SQP algorithm were required.

\begin{figure}[!h]
	\centering
	\includegraphics[width=.45\textwidth]{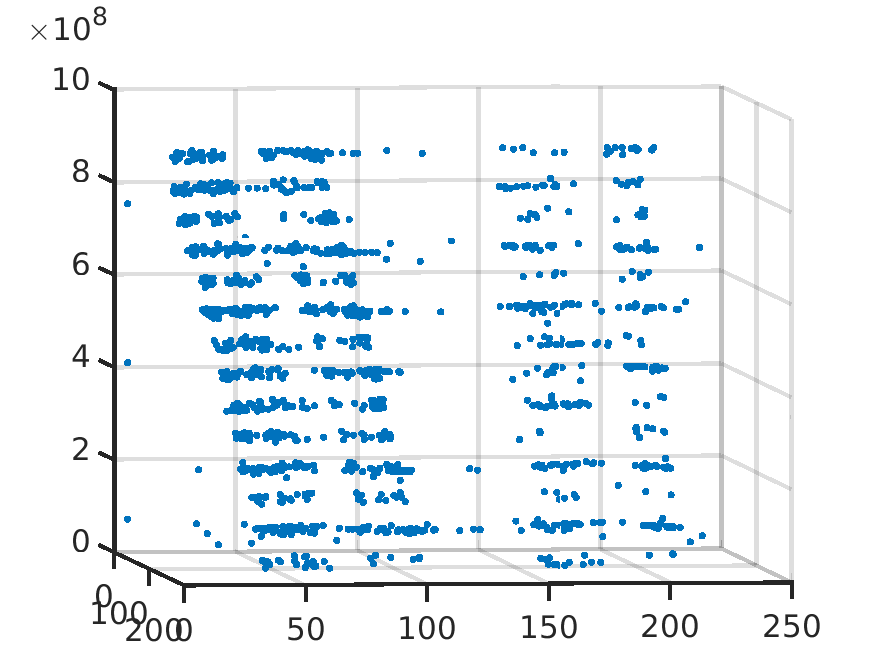}  \qquad
	\includegraphics[width=.45\textwidth]{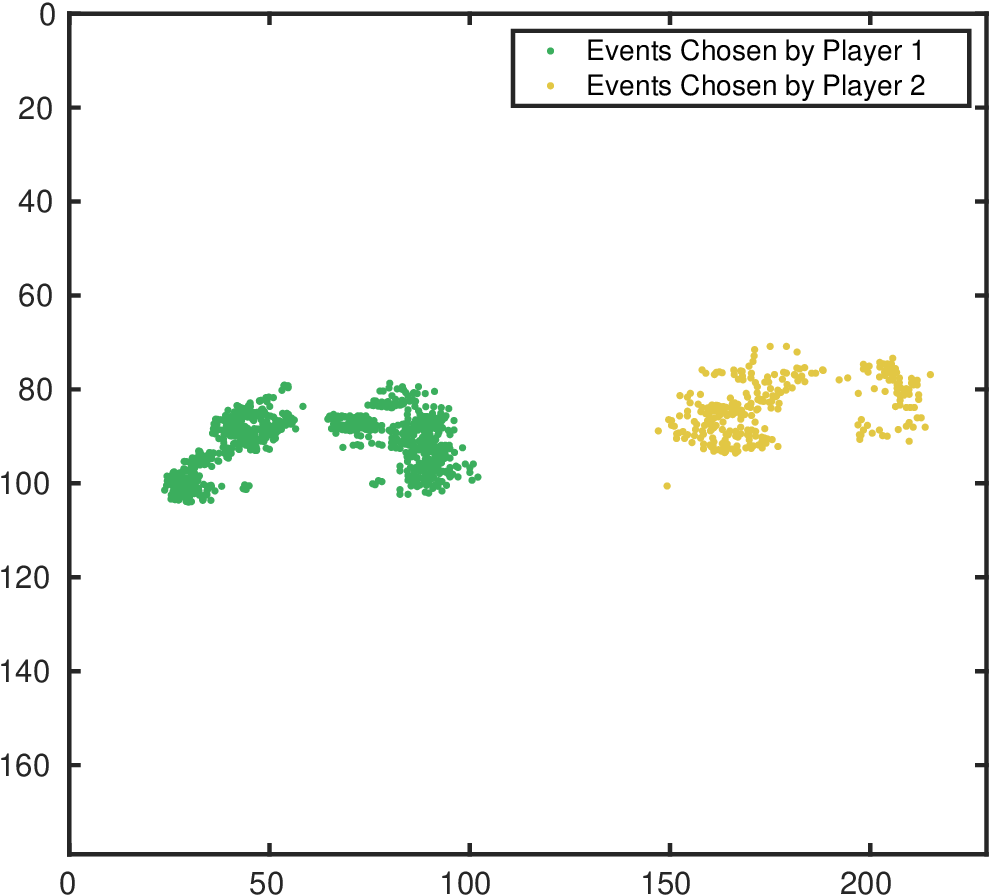}
	\caption{In this figure we show the processed data in the leftmost pane. The results of our method with two players is shown in the right pane. In this example $N_e$, the number of total events, is 1490, $\lambda_1 = 10^{-3}$, $\lambda_2 = 10$, $\beta = 0.1$ and $\alpha = 0.9$. }
	\label{f:ex2}
\end{figure}

Figure \ref{f:ex3} shows an additional multi-vehicle scenario with $N = 2$ players. For this example 2132 events were used. As shown in the right pane, our model is able to segment the image with only two pixels being classified incorrectly. These results required 211 iterations of the genetic algorithm and an additional 14 iterations of the SQP algorithm for player 1. Player 2 required 76 iterations of the genetic algorithm and 8 iterations of the SQP method.

\begin{figure}[!h]
	\centering
	\includegraphics[width=.45\textwidth]{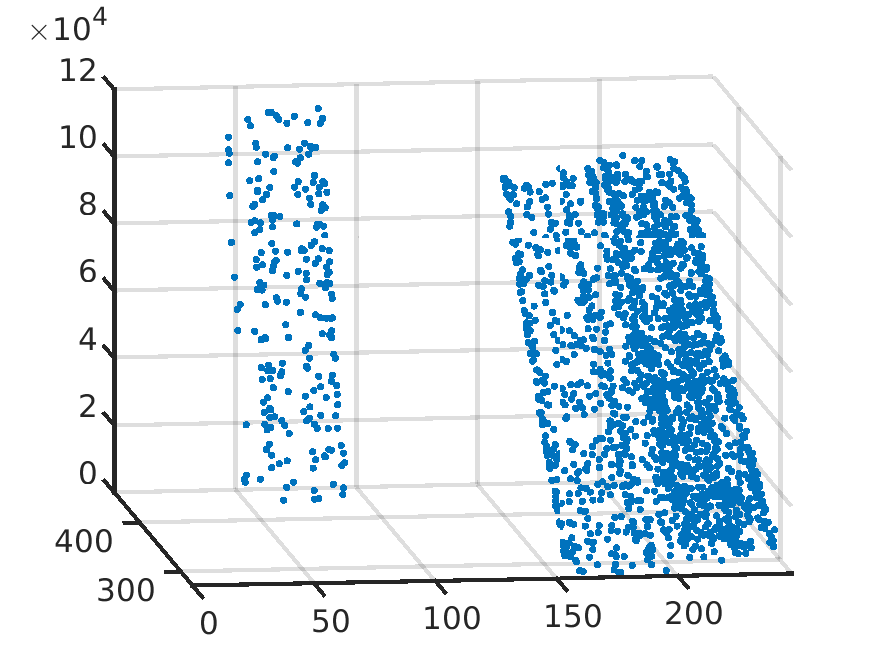}  \qquad
	\includegraphics[width=.45\textwidth]{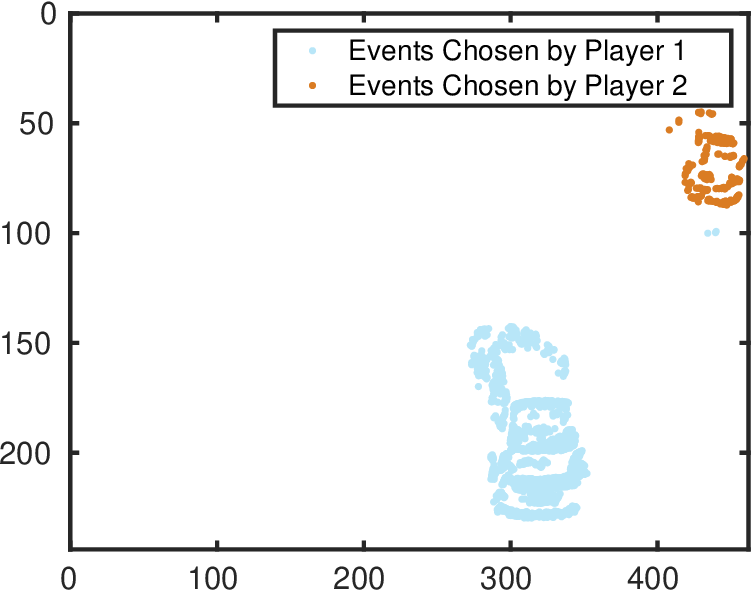}
	\caption{In this figure we show the processed data in the leftmost pane. The results of our method with two players is shown in the right pane. In this example $N_e$, the number of total events, is 2132, $\lambda_1 = 10^{-3}$, $\lambda_2 = 1.1$, $\beta = 0.1$ and $\alpha = 0.9$. }
	\label{f:ex3}
\end{figure}

As shown in Figure \ref{f:ex4} we have $N = 4$ players. These results used 1964 events with algorithm iterations shown in Table \ref{t:iter}. Left pane shows the processed input data, while the right pane shows our result. We note that there is slight overlap between a few pixels of players 1 and 4 and players 2 and 4 in this result.

\begin{figure}[!h]
	\centering
	\includegraphics[width=.45\textwidth]{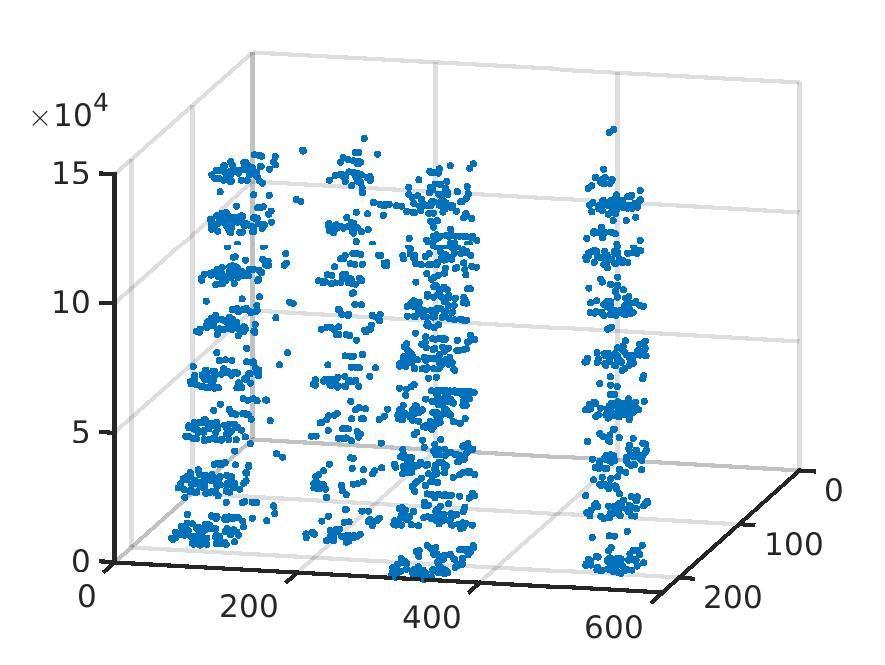}  \qquad
	\includegraphics[width=.45\textwidth]{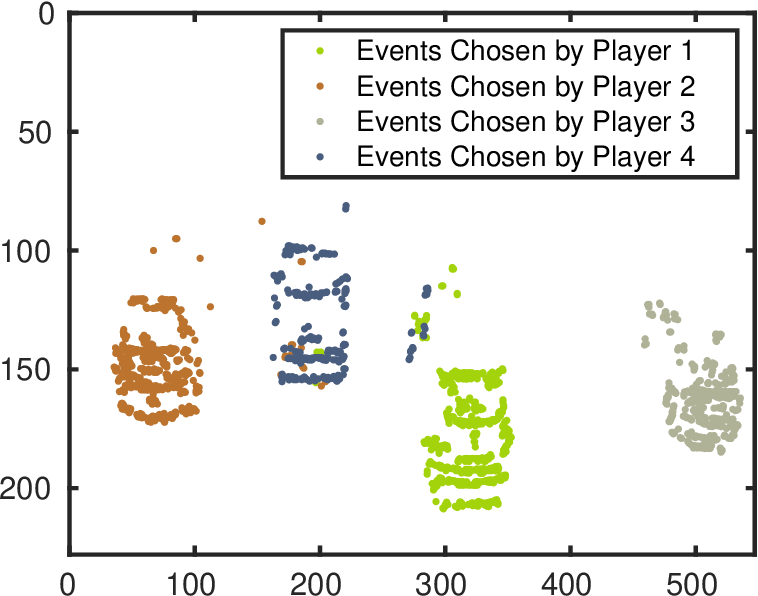}
	\caption{In this figure we show the processed data in the leftmost pane. The results of our method with four players is shown in the right pane. In this example $N_e$, the number of total events, is 1964, $\lambda_1 = 10^{-3}$, $\lambda_2 = 0$, $\beta = 0.1$ and $\alpha = 0.9$. }
	\label{f:ex4}
\end{figure}

\begin{table}[!h]
	\centering
	\caption{\label{t:iter}Algorithm Iterations for each player in the 4 player example.}
	\begin{tabular}{|c|c|c|}
		\hline
		& Genetic Algorithm  & SQP  \\
		\hline
		Player 1 & 583 & 7 \\
		\hline
		Player 2 & 455 & 14\\
		\hline
		Player 3 & 402 & 2\\
		\hline
		Player 4 & 269 & 2\\
		\hline
	\end{tabular}
\end{table}

\section*{Conclusion, Limitation, and Future Work}
In conclusion, this paper proposes a game theoretic model for the approximation of object velocities and segmentation of event data from a neuromorphic camera. The experimental results demonstrate the effectiveness of the proposed model in achieving accurate and robust segmentation of the event data. Furthermore, we have shown existence of solution to the Generalized Nash Equilibrium Problem and we have proposed a $N$-level optimization method for calculating equilibrium.

Though our method is able to generate equilibrium there are some limitations. 
For instance, manual tuning of $\lambda_1$, $\lambda_2$, and the highly nonconvex 
nature of the objective function leads to challenges for initial guesses for the iterates of
gradient based methods. In addition, we need to solve $N$-optimization problems. 
Finally, the dimension of the optimization variable 
$\bm{z}$ is equivalent to the number of events considered.

In future work we look to explore techniques that will allow us to use the convex hull of the objective function in order to use quasi-variational techniques to calculate equilibrium. Furthermore, we look to incorporate more advanced equations of motion for our warping function.

\section*{Acknowledgement}
The authors are grateful to Prof. Uday Shanbhag (Penn State) for inspiring us to explore 
the GNEP research area through various discussions and also sharing various references
on this topic. We are also thankful to Dr. Patrick O'Neil and Dr. Diego Torrejon (BlackSky) 
and Dr. Noor Qadri (US Naval Research Lab, Washington DC), for bringing the 
neuromorphic imaging topic to their attention and for several fruitful discussions.

\bibliographystyle{plain}
\bibliography{refs.bib}

\appendix 
\section{Visual Representation of Our Method} 
\label{s:app}

Figure~\ref{f:c_dat} illustrates the event data for two circular 
events over five timesteps. Player 1 and 2 respectively selects
the right and left events. This is shown in Figure~\ref{f:c_p1p2}. 
A projection mechanism, from 3D to 2D along the lines with 
slopes $\bm \theta^1$ and $\bm \theta^2$, is shown in 
Figure~\ref{f:c_projlines}. The projected data itself is given in 
Figure~\ref{f:c_proj2}.

\begin{figure}[!htb]
	\centering
	\includegraphics[width=\textwidth]{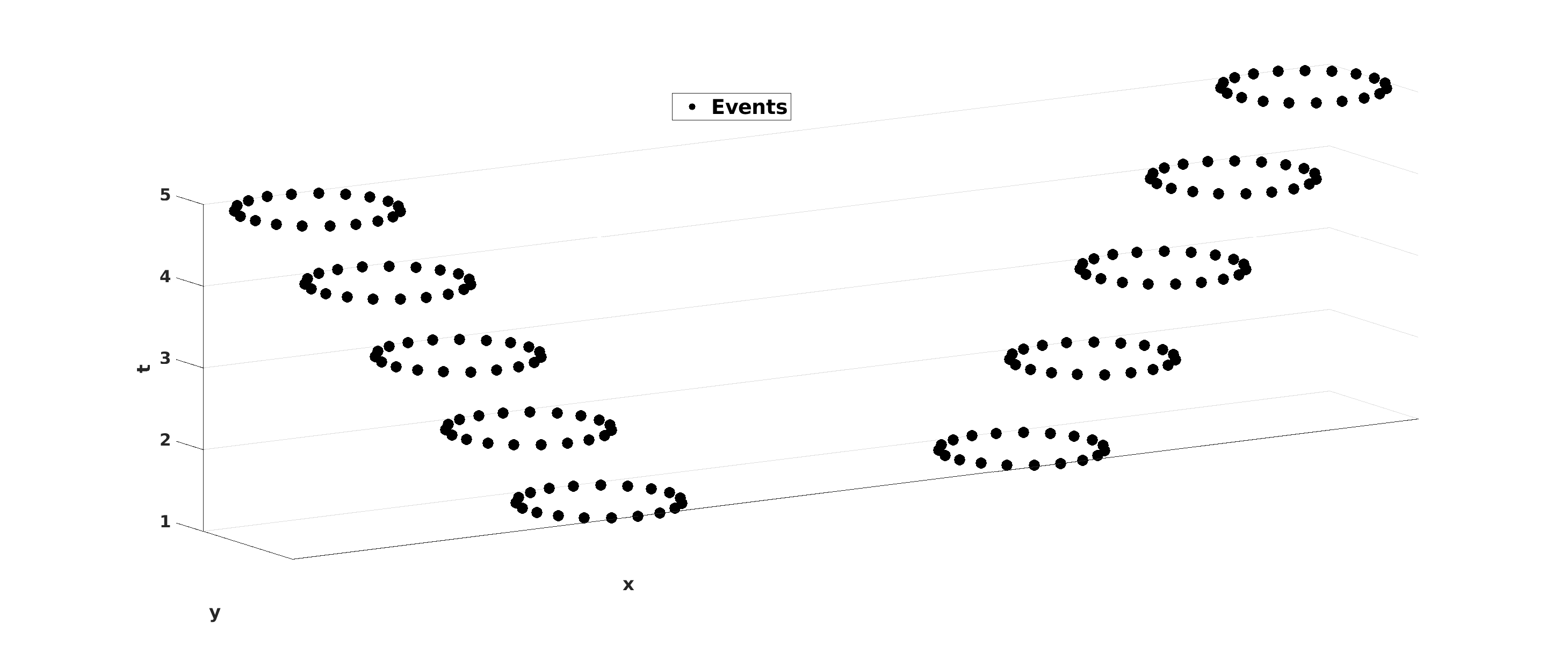}  \qquad
	\caption{In this figure we plot simulated event data of two circular objects over five timesteps.}
	\label{f:c_dat}
\end{figure}

\begin{figure}[!htb]
	\centering
	\includegraphics[width=\textwidth]{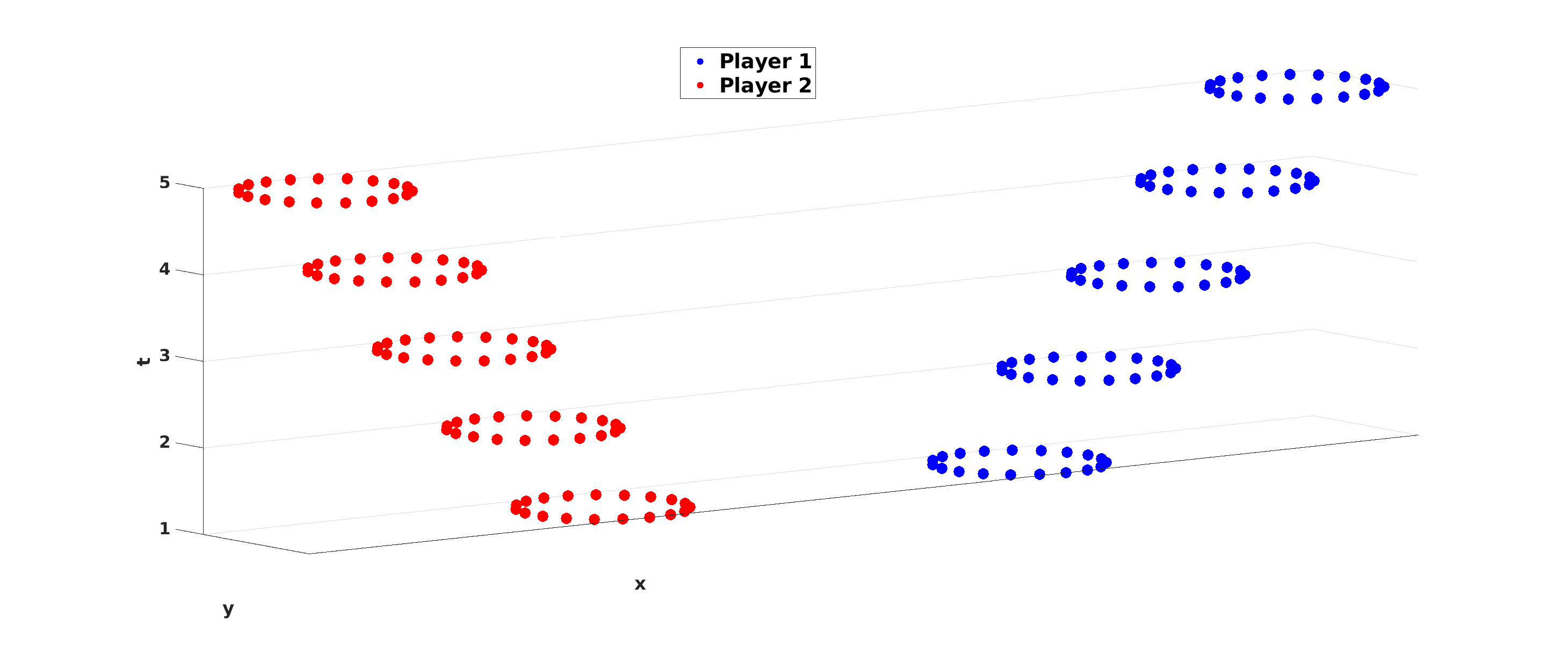}  \qquad
	\caption{In this figure we give a visual representation of the ideal strategies of Player 1 (blue) 
	and Player 2 (red), respectively for the event data given in Figure \ref{f:c_dat}.}
	\label{f:c_p1p2}
\end{figure}

\begin{figure}[!htb]
	\centering
	\includegraphics[width=\textwidth]{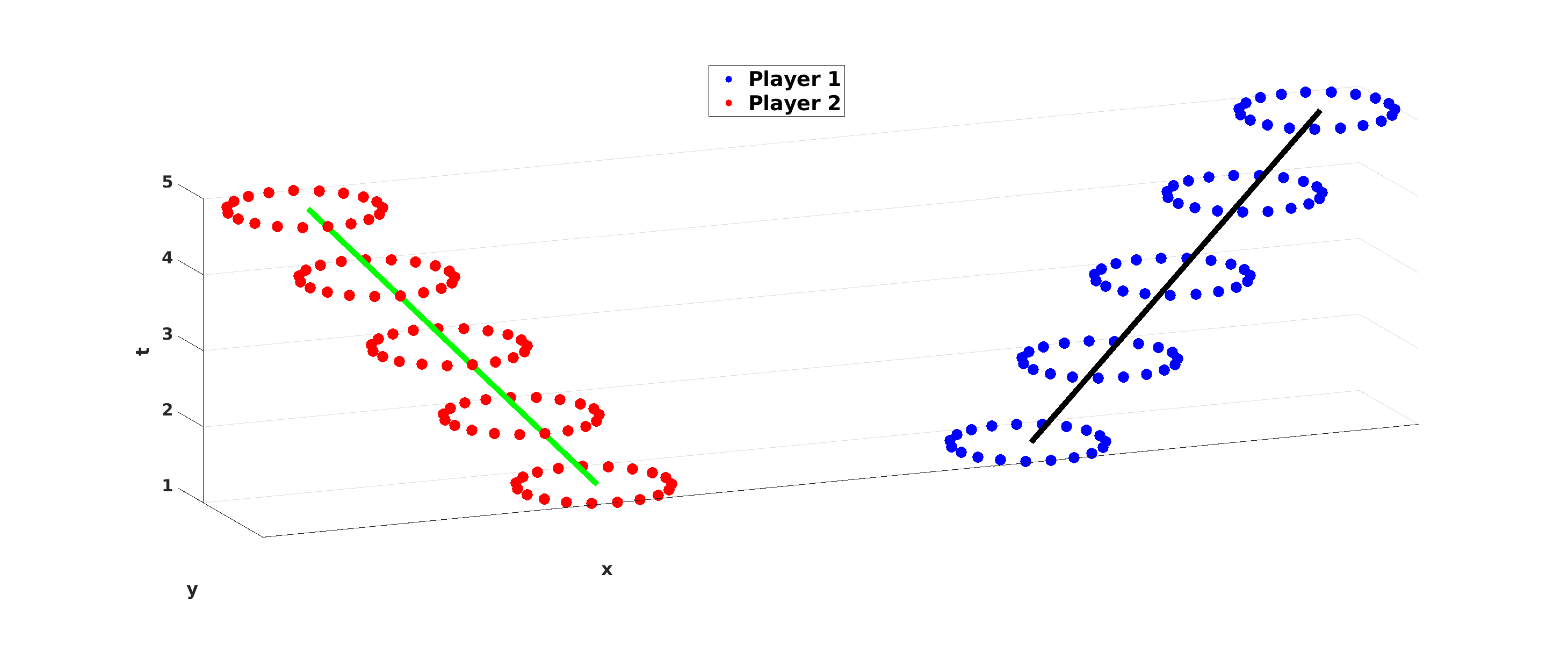}  \qquad
	\caption{In this figure we show the events chosen by Player 1 and Player 2 given in blue and red, respectively. Additionally, we plot projection lines given in black for player 1 and in green for player 2 that have slopes $\bm{\theta}^1$ and $\bm{\theta}^2$ respectively.}
	\label{f:c_projlines}
\end{figure}

\begin{figure}[!htb]
	\centering
	\includegraphics[width=\textwidth]{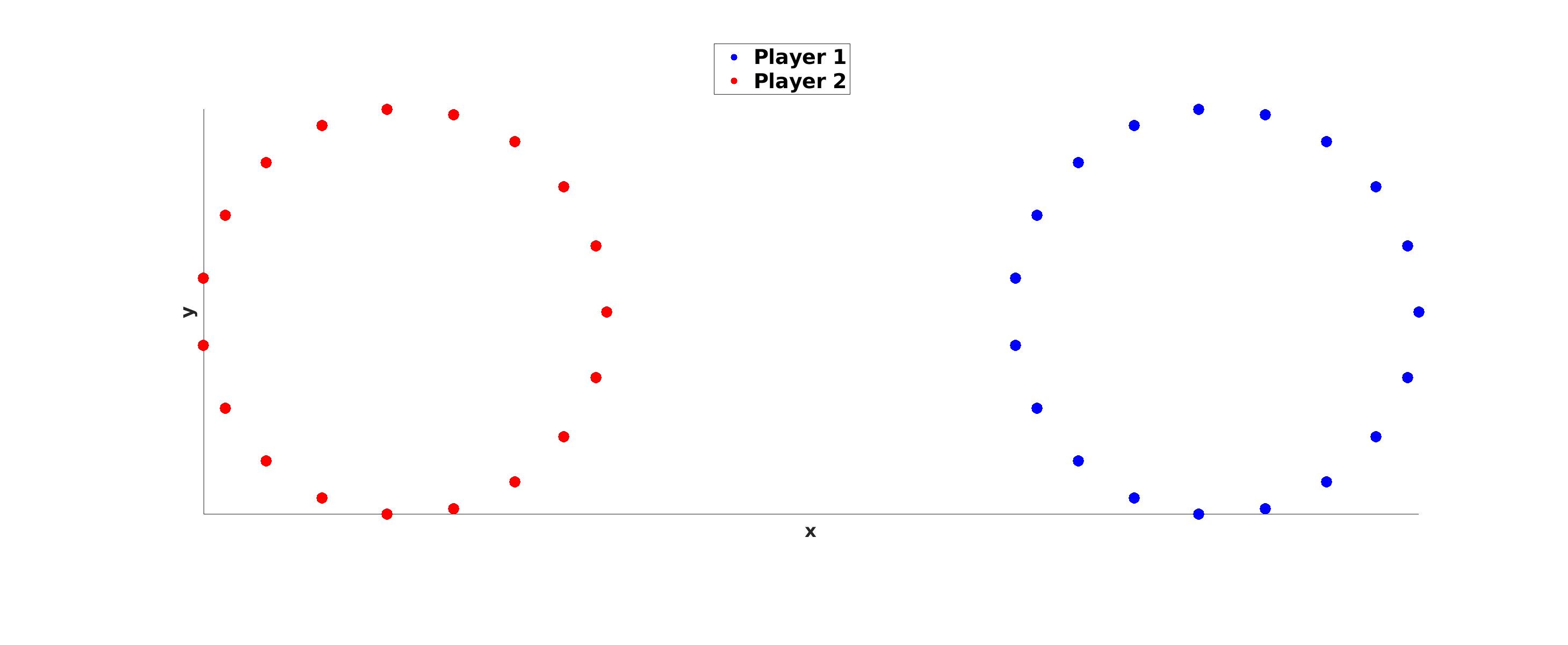}  \qquad
	\caption{In this figure we show the result of our method when applied to the data shown 
	in Figure~\ref{f:c_dat} with strategies of Player 1 and Player 2 shown in Figure \ref{f:c_p1p2}.}
	\label{f:c_proj2}
\end{figure}

\end{document}